\documentclass[]{bytedance_seed}

\usepackage[toc,page,header]{appendix}

\usepackage{minitoc}

\usepackage{amsmath,amsfonts,bm}

\def\eqref#1{equation~\ref{#1}}

\def\1{\bm{1}}

\newcommand{\NC}{\mathsf{NC}}

\newcommand{\TC}{\mathsf{TC}}
\newcommand{\T}{\mathsf{T}}

\def\ve{{\mathbf{e}}}

\def\vx{{\mathbf{x}}}

\newcommand{\ours}{\texttt{LED}}

\def\ve{{\bm{e}}}

\def\vx{{\bm{x}}}

\DeclareMathAlphabet{\mathsfit}{\encodingdefault}{\sfdefault}{m}{sl}
\SetMathAlphabet{\mathsfit}{bold}{\encodingdefault}{\sfdefault}{bx}{n}

\usepackage{url}

\usepackage{multirow}
\usepackage{multicol}
\usepackage{amsmath}
\usepackage{adjustbox}
\usepackage{verbatim}
\usepackage{amssymb}
\usepackage{mathtools}
\usepackage{caption}
\usepackage{subcaption}
\usepackage{array}
\usepackage{colortbl}
\usepackage{booktabs}
\usepackage{bbm}
\usepackage{makecell}
\usepackage{float}
\usepackage{siunitx}
\usepackage{pifont}
\usepackage{marvosym}
\usepackage{listings}
\usepackage{pdflscape}
\usepackage{url}            
\usepackage{booktabs}       
\usepackage{amsfonts}       
\usepackage{nicefrac}       
\usepackage{microtype}      
\usepackage{xcolor}         
\usepackage{tabularx}
\usepackage{amsmath, amssymb}
\usepackage{algorithm}
\usepackage{algpseudocode}
\usepackage{mathrsfs}
\usepackage{natbib}
\usepackage{booktabs}
\usepackage{colortbl}
\usepackage{multirow}
\usepackage{mathtools}
\usepackage{wrapfig} 
\usepackage{subcaption}
\usepackage{footmisc}
\usepackage{amsmath,amssymb,amsthm}
\usepackage{algorithm}
\usepackage{algpseudocode}
\usepackage{graphicx}

\newtheorem{theorem}{Theorem}[section]

 \newcommand{\scoregain}[2]{%
  \makebox[4.2em][c]{%
    #1\rlap{\scriptsize\textcolor{green!60!black}{$_{\textup{+#2}}$}}%
  }%
}

\title{\ours{}: Improving Masked Diffusion Language \\ Models by Lookahead Editing}
\author[1,2,*]{Yuyan Zhou}
\author[2,3,4,*]{Zeyuan Yang}
\author[2,\dagger]{Yuxuan Song}
\author[2,3,4]{Zhicheng Cai}
\author[2,3,4]{Hongli Yu}
\author[3,4]{Weiying Ma}
\author[3,4, \dagger]{Hao Zhou}
\author[1]{James Kwok}

\affiliation[1]{Department of Computer Science and Engineering,
The Hong Kong University of Science and Technology}

\affiliation[2]{ByteDance Seed}

\affiliation[3]{Institute for AI Industry Research (AIR), Tsinghua University}

\affiliation[4]{SIA-Lab of Tsinghua AIR and ByteDance Seed}

\contribution[*]{Equal contribution}
\contribution[\dagger]{Corresponding authors}

\abstract{
    Diffusion large language models (dLLMs) have emerged as a promising alternative to autoregressive language models, enabling non-autoregressive text generation through parallel decoding.
    Existing dLLMs can be broadly categorized into two families: masked diffusion models (MDMs)
    and uniform diffusion models (UDMs)
    In this work, we theoretically show that MDMs are easier to optimize due to their mask-based transition schedule but have lower computational expressivity, whereas UDMs enable iterative state revision and greater expressivity but are substantially harder to train.
    This trade-off raises a natural question: can we obtain the expressivity of iterative revision without giving up the simple, easy-to-optimize transition schedule of MDMs?
    Motivated by this analysis, we propose Lookahead Editing (\ours{}), which augments masked diffusion decoding with a rewritable proposal state. At each decoding step, the model performs several lookahead updates to iteratively refine future predictions in this proposal state before committing tokens to the masked sequence. 
    This design introduces a rewritable proposal that uses additional compute to improve expressivity while keeping the simple mask-based transition schedule of MDMs.
    We fine-tune Qwen2.5-7B and Qwen3-8B to obtain \ours{}-\texttt{7B} and \ours{}-\texttt{8B}, respectively. Across eight benchmarks, \ours{} SFT yields average gains of $+11.4$ and $+7.7$ points over standard MDM SFT. Moreover, using SFT alone, \ours{}-\texttt{8B} outperforms LLaDA2.1-16B Q-mode, which is trained with continued pre-training, supervised fine-tuning, and reinforcement learning.
}

\date{\today}

\checkdata[Github]{\href{https://github.com/Yzyy99/LED}{https://github.com/LED} }

\begin{document}

\maketitle

\section{Introduction}
Auto-regressive models (ARMs)~\citep{achiam2023gpt, liu2024deepseek, dubey2024llama} have achieved remarkable performance across a wide range of natural language tasks. 
However, ARMs are inherently limited by sequential decoding, as tokens are generated one at a time in a strict left-to-right order.
This makes decoding largely memory-bound, preventing ARMs from fully exploiting the parallelism and rapidly growing compute of modern hardware.
Diffusion-based language models (dLLMs)~\citep{nie2025large, zhu2025llada, khanna2025mercury,song2025seed} provide an alternative formulation of sequence generation that relaxes this fixed ordering and enables parallel denoising over discrete states.

Existing dLLMs can be broadly categorized into two families: masked diffusion models (MDMs)~\citep{nie2025large, ye2025dream}
and uniform diffusion models (UDMs)~\citep{lou2023discrete, team2026diffusiongemma}.
MDMs introduce a dedicated mask token and generate sequences by progressively replacing masked positions with predicted tokens, while UDMs define transitions over the entire vocabulary, allowing each token to be repeatedly replaced by other vocabulary tokens during denoising.
Despite sharing the same diffusion-based formulation, MDMs and UDMs exhibit fundamentally different properties in terms of {training difficulty} and {computational expressivity}.
Following the path-space ELBO decomposition of \cite{amin2026masking} for dLLMs,
which separates the objective into a transition prediction loss and a schedule matching loss, we show that UDMs incur an additional model-dependent schedule matching loss because their transition schedule depends on the denoiser. 
In contrast, the transition schedule~\cite{amin2026masking} in standard MDMs is independent of the denoiser and subsequently the schedule matching term is zero. 
This simplifies the learning objective and leads to a smaller excess loss than UDMs.

On the other hand, the training advantage of MDMs
comes with a limitation in computational expressivity. We first prove that MDMs remain within
$\TC_0$~\citep{vollmer1999introduction},  
a  
parallel computation model class 
in which the computational depth remains $O(1)$, 
even with arbitrary-order prediction and bidirectional attention, sharing the same complexity class with ARMs~\citep{li2024chain}.
UDMs, by contrast, can repeatedly rewrite previously generated tokens. We show that this rewritable state enables them to achieve $\TC_0$ and support $\NC_1$ computation, which can represent hierarchical computations of logarithmic depth such as evaluating nested Boolean expressions.

Motivated by the above analysis, we propose Lookahead Editing (\ours{}), which introduces a rewritable proposal that uses additional compute to improve expressivity while preserving the simple mask-based transition schedule of MDMs.
For training, we condition \ours{} on proposals generated by a pretrained MDM and optimize a standard cross-entropy loss. 
Since the proposal sequence generated by the pretrained MDM may contain both correct and incorrect tokens, the model learns when to preserve correct proposal tokens and when to revise incorrect ones.
At inference time, \ours{} iteratively alternates between a {proposal stage} and a {commit stage}.
The proposal stage maintains a rewritable intermediate state in the proposal branch, allowing tokens to be repeatedly generated and revised over multiple lookahead steps.
The refined proposal is then used as additional condition in the commit stage, which follows the standard masked diffusion process: once a token is committed, it is fixed and cannot be rewritten in subsequent decoding steps.
Across commit stages, \ours{} preserves reliable proposal tokens while revising erroneous ones to improve subsequent proposals.

We fine-tune Qwen2.5-7B and Qwen3-8B on the same training set using standard MDM SFT and \ours{} SFT.
Across eight benchmarks in math, code, and science, \ours{} SFT yields average gains of $+11.4$ and $+7.7$ points over standard MDM SFT for Qwen2.5-7B and Qwen3-8B, respectively.
Moreover, using SFT alone, \ours{}-\texttt{8B} outperforms LLaDA2.1-16B Q-Mode, which is trained with continued pre-training, supervised fine-tuning, and reinforcement learning.

\section{Preliminary}
\subsection{Diffusion Language Models}

Diffusion models have demonstrated remarkable success in continuous data domains~\citep{sohl2015deep, ho2020denoising, karras2022elucidating}. Building on these foundations, 
diffusion language models
have been developed to train discrete diffusion models directly within the categorical vocabulary space~\citep{sohl2015deep, austin2021structured, lou2023discrete, nie2025large, ye2025dream}.
Let $\mathcal{V}$ denote a finite vocabulary and 
$
X=(X_t)_{t\in[0,1]}
$
denote a Markov trajectory, where
$
X_t\in\mathcal{V}
$
is the state at diffusion time
$
t\in[0,1].
$
We denote the clean and terminal noisy states by the random variables $X_0$ and $X_1$, respectively, and their realizations by $\vx^0$ and $\vx^1$. 
The predefined {forward} noising process progressively corrupts
$
\vx_0
$
to
$
\vx_1
$,
and the learned {reverse} generative process progressively denoises 
$
\vx^1
$
to
$
\vx^0
$:
\[
p(X_t\mid X_s)
=
\operatorname{Cat}\!\left(X_t;Q_{s,t}X_s\right),\quad
q_\theta(X_s\mid X_t)
=
\operatorname{Cat}\!\left(X_s;Q^\theta_{t,s}X_t\right),
\quad
0\le s<t\le1,
\]
where
$
Q_{s,t}
$
is the forward transition kernel from time
$
s
$
to
$
t
$,
inducing the path distribution
$
p(X\mid \vx^0),
$
and
$
Q^\theta_{t,s}
$
is the model-dependent reverse transition kernel, inducing the reverse path distribution
$
q_\theta(X),
$
where $\theta$ is the model parameter.
MDMs and UDMs differ in their choice of the forward transition kernel:
MDMs use an absorbing transition toward a dedicated mask token $m$,
whereas UDMs transition toward the uniform distribution over the
vocabulary~\citep{austin2021structured},
\[
Q^{\mathrm{MDM}}_{s,t}
=
\alpha_{s,t} \mathbf{I}
+
(1-\alpha_{s,t}) \ve_m\mathbf{1}^{\top},
\qquad
Q^{\mathrm{UDM}}_{s,t}
=
\alpha_{s,t} \mathbf{I}
+
(1-\alpha_{s,t})
\frac{1}{|\mathcal V|}
\mathbf{1}\mathbf{1}^{\top},
\]
where $\ve_m$ denotes the one-hot vector of the mask token, $\mathbf{I}$ is identity matrix and
$\alpha_{s,t}$ controls the noise level between diffusion times $s$ and $t$.
The model parameters are learned by maximizing an evidence lower bound (ELBO) on the data log-likelihood,
\begin{align}
\label{eqn:elbo}
\mathcal{L}_{\mathrm{ELBO}}(\theta)
=
\mathbb{E}_{ p_{\textup{data}}(\vx^0)}
\mathbb{E}_{ p(X\mid \vx^0)}
\left[
\log
\frac{q_\theta(X)}
     {p(X\mid \vx^0)}
\right]
\le
\mathbb{E}_{ p_{\textup{data}}(\vx^0)}
\left[
\log q_\theta(\vx^0)
\right].
\end{align}
Instead of corrupting and generating tokens uniformly over the full sequence length, recent SOTA dLLMs adopt block diffusion language models (BDLMs)~\citep{arriola2025block, wu2025fast2, cheng2025sdar, bie2026llada2}, which 
executes discrete diffusion
by partitioning the sequence into multiple blocks.
Unless otherwise stated, our discussion throughout this paper focuses on BDLMs.
\subsection{Circuit Complexity} 
Circuit complexity~\citep{vollmer1999introduction}
provides a standard framework for characterizing the computational expressivity of parallel models.
A Boolean circuit is a directed acyclic graph composed of logic gates and is characterized primarily by its size, depth, fan-in, and gate types.
Recently,
circuit complexity has 
been used to characterize the expressivity of language models~\citep{merrill2023parallelism,liu2022transformers, li2024chain}.
They consider three families of circuit classes: $\mathsf{NC}^k$, $\mathsf{AC}^k$, and $\mathsf{TC}^k$, 
where $k\ge 0$.
For an input of length $n$, $\mathsf{NC}^k$ consists of problems computable by polynomial-size, $O((\log n)^k)$-depth circuits with bounded-fan-in AND, OR, and NOT gates.
$\mathsf{AC}^k$ is defined similarly but allows AND and OR gates with unbounded fan-in, while $\mathsf{TC}^k$ further allows unbounded-fan-in MAJORITY gates. 
These classes satisfy
$
\mathsf{NC}^k
\subseteq
\mathsf{AC}^k
\subseteq
\mathsf{TC}^k
\subseteq
\mathsf{NC}^{k+1}$.

With polynomially bounded hidden dimension, constant-depth autoregressive transformers without additional intermediate computation are contained in $\mathsf{TC}^0$ under logarithmic numerical precision; and in $\mathsf{AC}^0$ under constant numerical 
precision. 
Allowing $O(\log n)$ chain-of-thought steps does not increase these asymptotic upper bounds. The resulting models remain in $\mathsf{TC}^0$ under logarithmic fixed-point precision, and in $\mathsf{AC}^0$ under constant precision~\citep{li2024chain}.

\section{Optimization-Expressivity Trade-off}
\label{sec:scheduling_rewriting}

MDM and UDM exhibit markedly different behaviors in both training and generation. 
MDM benefits from the explicit mask state, which provides a strong inductive bias for learning, but its generation process is non-rewriting: once a token is committed, it cannot be revised. 
UDM, in contrast, allows intermediate states to be repeatedly updated, offering a potentially richer computational mechanism, but removes the explicit masking structure that simplifies transition learning. 
In this Section, we extend the circuit-complexity perspective to study how the decoding dynamics of diffusion language models affect their computational expressivity.
\subsection{Masking Eliminates Schedule Learning.}
For any path
$X$, let
$
M(X)
$
be the number of state transitions along the trajectory. Define
the {transition schedule}
$
S(X)
=
\{t_1,\ldots,t_{M(X)}\}$,
where 
$0<t_1<\cdots<t_{M(X)}<1$, and
$X$ changes state
at time 
$t_j$.
Following \citet{amin2026masking}, the ELBO in \cref{eqn:elbo} can be decomposed as two losses:
$\mathcal{L}_{\textup{where}}(\theta)$
is about the state transitions (i.e., {  where to jump}), while 
$\mathcal{L}_{\textup{when}}(\theta)$
is about the transition schedule (i.e., {  when to jump}):
\begin{align}
\mathcal{L}_{\mathrm{ELBO}}(\theta)
={}&-
\underbrace{
\mathbb{E}_{p_{\textup{data}}(\vx^0)}
\mathbb{E}_{p(X\mid \vx^0)}
\left[
\log
\frac{p(X\mid \vx^0,\vx^1,S)}
     {q_\theta(X\mid \vx^1,S)}
\right]
}_{\mathcal{L}_{\textup{where}}(\theta)} - 
\underbrace{
D_{\mathrm{KL}}
\left(
p(S)\,\|\,q_\theta(S)
\right)
}_{\mathcal{L}_{\textup{when}}(\theta)}
+ C,
\label{eq:schedule_decomposition}
\end{align}
where
$C =  \mathbb{E}_{p_{\textup{data}}(\vx^0)p(S\mid \vx^0)}
\left[
\log\frac{p(S)}{p(S\mid \vx^0)}
\right] - \mathbb{E}_{p_{\textup{data}}(\vx^0)}
\mathbb{E}_{p(S\mid \vx^0)}
\left[
D_{\mathrm{KL}}
\left(
p(\vx^1\mid \vx^0,S)
\,\|\,q(\vx^1\mid S)
\right)
\right]$ is a constant independent of the model parameter $\theta$, and $q_\theta(S)$ is the pushforward distribution of $q_\theta(X)$, i.e. $q_\theta(S) =
\int
\mathbf{1}\!\left\{S(X)\in A\right\}
\,q_\theta(dX):=S_{\#} q_\theta(X)$.

The following Theorems 
show that 
the schedule-matching term 
vanishes for MDMs; while for UDMs it depends on the denoiser through the induced reverse jump rate.
\begin{theorem}
\label{thm:mdm_where}
For MDM,
$
q_\theta(S)=p(S),\ \forall\theta,
$ 
and thus
$
\mathcal{L}_{\mathrm{when}}(\theta)
=
D_{\mathrm{KL}}\!\left(p(S)\|q_\theta(S)\right)
=
0.
$
\end{theorem} 
\begin{theorem}
\label{thm:udm_schedule_matching} 
For UDMs, let $\tau=1-t$, $Y_\tau=X_{1-\tau}$, and 
$ \mathcal F_{\tau^-}^{S} := \{\tau_i\in S:\tau_i<\tau\}$. 
The schedule-matching loss is
$
\mathcal L_{\mathrm{when}}(\theta)
=
\mathbb E_{S,t} 
\left[
\mu^*(\tau) s
\log
\frac{
\mu^*(\tau) 
}{
\bar\mu_\theta (\tau\mid\mathcal F_{\tau^-}^{S})
}
-
\mu^*(\tau)
+
\bar\mu_\theta(\tau\mid\mathcal F_{\tau^-}^{S})
\right],
$
where
$\mu^*(\tau)
=
\frac{V-1}{V}\beta_{1-\tau},$
$
\beta_t = -\frac{d}{dt}\log \alpha_t,
$
$
V = |\mathcal V|,
$
$\bar\mu_\theta
(\tau\mid\mathcal F_{\tau^-}^{S})
=
\mathbb E_{q_\theta}
\left[
\lambda_{f_\theta}(Y_\tau,1-\tau)
\mid
\mathcal F_{\tau^-}^{S}
\right]$.
and
$
\lambda_{f_\theta}(x,t)
=
\frac{\beta_t}{V}
\left[
f_\theta(x\mid x,t)
(1-\alpha_t)
/(1/(V-1)+\alpha_t)
+
\bigl(1-f_\theta(x\mid x,t)\bigr)
(V-1+\alpha_t)/ (1-\alpha_t)
\right].
$
\end{theorem}

The where loss for MDMs and UDMs are implemented using the same cross-entropy (CE) objective~\citep{bie2026llada2,team2026diffusiongemma}, while the when loss is implemented using the following tractable schedule-matching objective
$
\mathbb E_{S,\tau,Y_\tau}
\left[
\mu^*(\tau)
\log
\frac{\mu^*(\tau)}
{\lambda_{f_\theta}(Y_\tau,1-\tau)}
-\mu^*(\tau)
+
\lambda_{f_\theta}(Y_\tau,1-\tau)
\right]
$
~\citep{campbell2022continuous}.
Let MDM and UDM both use SGD with the same learning rate $\eta$, as
$\theta_{k+1}
=
\theta_k
-
\eta
(
\nabla \mathcal L(\theta_k)
+
\xi_k
)$.
The following Theorem shows that with the same where loss, UDMs is harder to train in terms of admitting a smaller learning rate and resulting in a larger loss.

\begin{theorem}
\label{thm:sgd_convergence}
Assume that the gradient noise is unbiased (i.e., $\mathbb E[\xi_k\mid\theta_k]=0$) and the model is sufficiently expressive to represent the data distribution, i.e. $q_{\theta^\star}=p_{\textup{data}}.$
MDM admits a larger learning rate than UDM
near the optimum under second order Taylor approximation, 
i.e.,
$\eta_{\max}^{\text{MDM}}=2/\lambda_{\max}(H_M)\ge \eta_{\max}^{\text{UDM}}=2/\lambda_{\max}(H_U)$.
Moreover, for
learning rate satisfying $0<\eta<2/\lambda_{\max}(H_U)$, the  expected excess losses satisfy
$\mathbb E \left[\mathcal L_U-\mathcal L_U^\star\right]
\ge
\mathbb E \left[\mathcal L_M-\mathcal L_M^\star\right]$. Since $
\mathcal L_U^\star
\ge
\mathcal L_M^\star,
$
it follows that 
$
\mathbb E[\mathcal L_U]
-
\mathbb E[\mathcal L_M]
\ge
\mathcal L_U^\star
-
\mathcal L_M^\star
\ge 0.
$
\end{theorem}

\subsection{Rewriting improves Expressivity.}
The optimization advantage of masking, however, comes with a structural restriction. Once a position is unmasked and committed, standard
masked diffusion does not revise it in subsequent decoding steps. 
We next study the computational consequence of this monotone generation process.

Following \citet{li2024chain}, we define
$
\T_{\mathrm{ARM}}[T(n),d(n),s(n),e(n)]
$
as the function class of autoregressive transformers, where the four parameters specify, respectively, the number of CoT steps, the hidden dimension, the numerical precision in bits, and the number of exponent bits. 
\citet{li2024chain} show that autoregressive transformers with polynomial hidden dimension,
$
O(\log n)
$-bit numerical precision, and fixed-point arithmetic remain contained in
$
\mathsf{TC}^0
$
even with
$
O(\log n)
$
CoT steps.

\begin{theorem}[\citet{li2024chain}]
$
\T_{\mathrm{ARM}}[0,\operatorname{poly}(n),\log n,0]
\subseteq
\T_{\mathrm{ARM}}[\log n,\operatorname{poly}(n),\log n,0]
\subseteq
\mathsf{TC}^0.
$
\end{theorem}

Denote 
the function class of MDMs
by
$
\T_{\mathrm{MDM}}[T(n),d(n),s(n),e(n),B(n)]
$.
The first four parameters are the same as
$\T_{\mathrm{ARM}}$,
Compared with autoregressive models, MDMs introduce two additional sources of decoding flexibility: bidirectional attention within each block and arbitrary-order generation. 
The following Theorem shows that, for a constant block size, these additional mechanisms are still insufficient to move the model beyond 
$
\mathsf{TC}^0
$.

\begin{theorem}
\label{thm:mdm_experss}
$
\T_{\mathrm{MDM}}[0,\operatorname{poly}(n),\log n,0,O(1)]
\subseteq
\T_{\mathrm{MDM}}[\log n,\operatorname{poly}(n),\log n,0,O(1)]
\subseteq
\mathsf{TC}^0.
$
\end{theorem}
$
\T_{\mathrm{MDM}}[0,\operatorname{poly}(n),\log n,0,O(1)]
$
above
isolates the effect of bidirectional attention by removing CoT, 
$
\T_{\mathrm{MDM}}[\log n,\operatorname{poly}(n),\log n,0,O(1)]$ 
permits
$
O(\log n)
$
CoT steps with arbitrary-order generation.
However, they are still insufficient 
to provide a fundamentally stronger computational resource than autoregressive decoding.

The situation changes when previously generated states can be repeatedly rewritten. Let
$\T_{\mathrm{UDM}}[T(n), \allowbreak d(n), \allowbreak s(n), \allowbreak e(n), \allowbreak W(n)]$
be the function class of UDMs, where the first four parameters are the same as
$\T_{\mathrm{ARM}}$ and
$\T_{\mathrm{MDM}}$, while
the last parameter $W(n)$ is the size of the rewritable workspace (i.e., CoT length).
UDMs allow polynomially many revisions even with a constant-size workspace. The following Theorem shows that with polynomially many revisions over a constant workspace, UDMs can achieve $\mathsf{NC}^1$.
\begin{theorem}
\label{thm:udm_experss}
$
\mathsf{NC}^1
\subseteq
\T_{\mathrm{UDM}}
[\operatorname{poly}(n),\operatorname{poly}(n),\log n,0,O(1)].
$
\end{theorem}

\begin{figure}[t]
    \centering
    \includegraphics[width=0.8\linewidth]{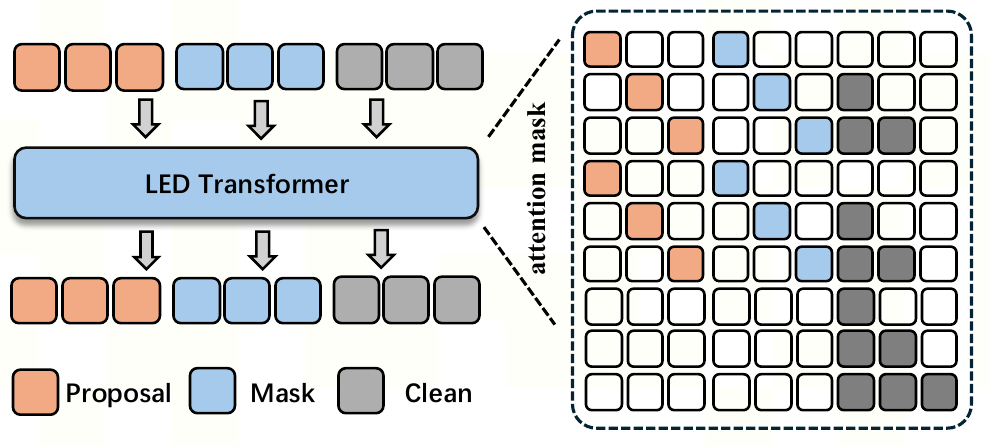}
    \caption{ 
    Training of Lookahead Editing.
    The proposal, masked, and clean branches are concatenated and processed by the \ours{} Transformer in a single forward pass.
    The structured attention mask controls information flow across the three branches, where rows denote query tokens and columns denote key/value tokens.
    Within each proposal mask block pair, the two branches attend bidirectionally to each other and to all preceding clean blocks, while the clean branch follows block-causal attention.
    }
    \label{fig:training}
\end{figure}

\section{Proposed Method: Lookahead Editing} 
While UDMs offer strong expressivity through iterative state revision, their denoiser-dependent transition schedules hinder optimization. 
MDMs benefit from simple, schedule-free optimization but suffer from constrained expressivity due to non-rewritable decoding.
To bridge this gap, we propose Lookahead Editing (\ours{}), a simple yet effective framework that introduces a rewritable proposal state into masked diffusion decoding.
We detail its training and inference procedures below.

\subsection{Training} 
Traditional MDMs adopt a two-branch optimization framework consisting of a masked branch and a clean branch~\citep{wu2025fast2, cheng2025sdar, fu2026nemotron}.
To enable editing capabilities in MDMs, we add a proposal branch that yields both correct and erroneous tokens, allowing the model to identify erroneous tokens and subsequently correct them while retaining correct tokens.
The pseudocode is shown in \cref{alg:training}.
For each training sequence, we construct a three-branch input consisting of:
(i) a clean branch representing the ground-truth sequence $\mathtt{X}_0$; 
(ii) a masked branch $\mathtt{X}_t$, obtained by applying standard masking \citep{arriola2025block} to $\mathtt{X}_0$; and 
(iii) a proposal branch $\widehat{\mathtt{X}}_0$, generated by a pretrained masked block diffusion model $f_{\textup{MDM}}$ conditioned on the concatenated clean and masked branches $[\mathtt{X}_0, \mathtt{X}_t]$.
We concatenate the clean branch $\mathtt{X}_0$, masked branch $\mathtt{X}_t$, and proposal branch $\widehat{\mathtt{X}}_0$ as input to $f_{\ours{}}$.
The attention mask is illustrated in \cref{fig:training}. 
Within each mask–proposal block pair, the two blocks attend bidirectionally to each other as well as to all preceding clean blocks, whereas attention within the clean branch remains strictly block-causal.
This design enables bidirectional interactions between mask and proposal blocks, while preventing information leakage from future clean tokens.
The model outputs $\mathtt{X}_{\mathrm{logit}}$ corresponding to the masked branch are optimized via a cross-entropy loss against the clean target sequence $\mathtt{X}_0$.
Because the proposal tokens generated by $f_{\textup{MDM}}$ are a mix of correct predictions and errors, they provide natural supervision for editing.
Specifically, the loss encourages the model to selectively retain accurate proposal tokens while updating inaccurate ones.

\begin{figure*}[t]
\centering

\begin{minipage}[t]{0.485\textwidth}
\hrule
\vspace{2pt}
\refstepcounter{algorithm}
\label{alg:training}
\noindent
\textbf{Algorithm 1} \ours{} Training.
\vspace{2pt}
\hrule
\vspace{4pt}

\begin{algorithmic}
\State \textbf{Input:} clean sequence $\mathtt{X}_0$, forward noise process
$q(\cdot\mid\mathtt{X}_0, t)$, pretrained model $f_{\mathtt{MDM}}$,
model $f_{\ours{}}$, loss $\mathcal{L}_{\mathrm{CE}}$

\Repeat
    \State Sample $t \sim \mathcal{U}[0,1]$
    \State $\mathtt{X}_t
    \sim p(\cdot\mid\mathtt{X}_0, t)$

    \State $\widehat{\mathtt{X}}_0
    \gets
    f_{\mathtt{MDM}}
    ([\mathtt{X}_0,\mathtt{X}_t])$

    \State $\mathtt{X}_{\mathrm{logit}}
    \gets
    f_{\ours{}}
    ([\mathtt{X}_0,\mathtt{X}_t,\widehat{\mathtt{X}}_0])$

    \State $\mathtt{X}_{\mathrm{logit}}
    \gets
    \mathtt{X}_{\mathrm{logit}}[\mathtt{X}_t]$

    \State Take gradient step on
    $\nabla_\theta
    \mathcal{L}_{\mathrm{CE}}
    (\mathtt{X}_{\mathrm{logit}};\mathtt{X}_0)$
\Until{converged}
\end{algorithmic}

\vspace{3pt}
\hrule
\end{minipage}
\hfill
%
\begin{minipage}[t]{0.485\textwidth}
\hrule
\vspace{2pt}
\refstepcounter{algorithm}
\label{alg:inference}
\noindent
\textbf{Algorithm 2} \ours{} Inference.
\vspace{2pt}
\hrule
\vspace{4pt}

\begin{algorithmic}
\State \textbf{Input:} model $f_{\ours{}}$,
lookahead steps $K$, sampler algorithm  $\mathcal{S}$

\State $\mathtt{X}^\texttt{T}, \mathtt{LX}^\texttt{t0}, \mathtt{C}^\texttt{T}
    \gets
    [\texttt{MASK}]^L$
{
    \algtext*{EndFor}
    \For{$t=T$ \textbf{to} $1$}

    \State $\mathtt{LC}^\texttt{t1}\gets f_{\ours{}}(\mathtt{p}, \mathtt{X}^t, \mathtt{C}^\texttt{t})$ \Comment{Step 1}
    \State $\mathtt{LX}^\texttt{t1}, \mathtt{LC}^\texttt{t1}\gets \mathcal{S}(\mathtt{LO}^\texttt{t0}), \mathtt{LO}^\texttt{t0}$

        \For{$k=1$ \textbf{to} $K$} \Comment{Step 2}
            \State $\texttt{LO}^\texttt{tk}
            \gets
            f_{\ours{}}
            (\mathtt{p},
            \texttt{LX}^\texttt{tk},
            \texttt{LC}^\texttt{tk})$
            \State $\texttt{LX}^\texttt{t,k+1}, \texttt{LC}^\texttt{t,k+1}\gets \mathcal{S}(\texttt{LO}^\texttt{tk}), \texttt{LO}^\texttt{tk}$
        \EndFor

        \State $\mathbf{C}^t
        \gets
        f_{\ours{}}
        ( \mathtt{p},
        \mathtt{X}^t,
         \texttt{LC}^\texttt{t,K+1})$ \Comment{Step 3}

        \State $\mathtt{X}^{t-1}, \mathtt{C}^{t-1}\gets \mathcal{S}(\mathtt{C}^t), \mathtt{C}^t$
    \EndFor
        }
    \Return $\mathtt{X}^0$
 \end{algorithmic}

\hrule
\end{minipage}
\caption{Training and inference procedures of Lookahead Editing. During training, a pretrained MDM generates the proposal, and \ours{} is trained on the resulting clean, noised, and proposal branches. During inference, \ours{} performs \(K\) iterative lookahead updates before each actual decoding step while retaining the original mask-based sampling schedule.}
\end{figure*}

\begin{figure}[t]
    \centering
    \includegraphics[width=1.0\linewidth]{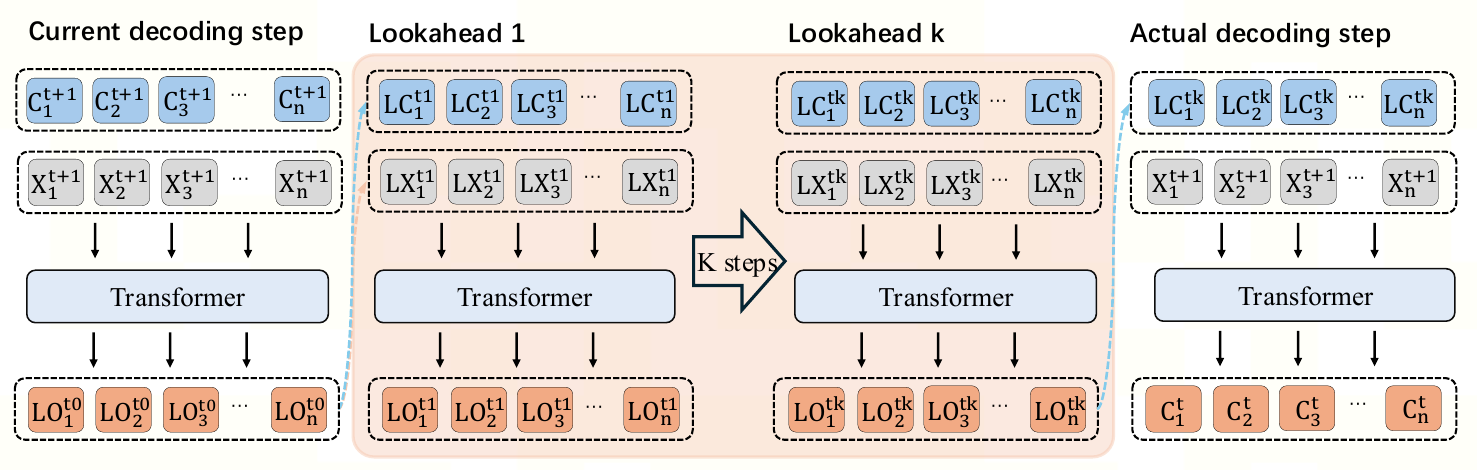}
\caption{
    Lookahead Editing decoding process. Orange arrow denotes and blue arrow denotes sample. At each decoding step, \ours{} performs \(K\) lookahead updates on the rewritable proposal state, allowing future predictions to be iteratively explored and refined. The resulting proposal is then used to perform the actual decoding update and advance the sequence to the next decoding step.
    }
    \label{fig:demo}
\end{figure}

\subsection{Inference}

The pseudocode and schematic overview of our inference procedure are provided in \cref{alg:inference} and \cref{fig:demo}, respectively. 
At decoding step $\mathtt{t}$, inference maintains three sequence components: prompt $\mathtt{p}$, a partially masked sequence $\mathtt{X}^{\mathtt{t}+1}$, and a rewritable proposal sequence $\mathtt{C}^{\mathtt{t}+1}$. 
Rather than immediately committing tokens from a single-step prediction, $f_{\ours{}}$ performs {Lookahead Editing} to iteratively refine future predictions prior to commitment:
\begin{enumerate}
    \item Initial proposal generation:
    At the beginning
    of step $\mathtt{t}$, the current proposal 
    $\mathtt{C}^{\mathtt{T}}$ and masked sequence $\mathtt{X}^{\mathtt{T}}$  are initialized as fully masked sequences and processed by $f_{\ours{}}$. 
    This yields an initial prediction $\mathtt{LO}^{\mathtt{t}_0}$, which is  copied to $\mathtt{LC}^{\mathtt{t1}}$ and serves as the base proposal for the subsequent lookahead iterations.
    \item Lookahead proposal refinement:  
    As illustrated in \cref{fig:demo},
    for each lookahead step $\mathtt{k} \in \{1, \dots, K\}$, the proposal block
    $\mathtt{LC}^{\mathtt{t}\mathtt{k}}$ is updated based on the model's output from the preceding step. 
    Passing $\mathtt{LX}^{\mathtt{t}{\mathtt{k}}}$ and $\mathtt{LC}^{\mathtt{t}{\mathtt{k}}}$ into $f_{\ours{}}$ yields updated predictions $\mathtt{LO}^{\mathtt{t}{\mathtt{k}}}$. 
    These predictions serve as the proposal sequence $\mathtt{LC}^{\mathtt{t}{\mathtt{k+1}}}$ for the subsequent lookahead step, while high-confidence tokens are sampled by sampler $\mathcal{S}$ into the temporary $\mathtt{LX}^{\mathtt{t}{\mathtt{k+1}}}$.
    This allows the draft proposal to be iteratively rewritten and explored across $K$ steps.
    
    \item {Decoding commitment:}  
    After $K$ lookahead steps, the highly refined proposal  $\mathtt{LC}^{\mathtt{t},{\mathtt{K+1}}}$ is combined with $\mathtt{X}^{\mathtt{t}-1}$ to compute token confidence scores. 
    The model then commits a subset of high-confidence tokens in $\mathtt{C}^{\mathtt{t}}$ to get  $\mathtt{X}^{\mathtt{t-1}}$, advancing the sequence to the next decoding step.
\end{enumerate}

Decoding alternates between rewritable lookahead drafting and irreversible token commitment until the current block is fully decoded, after which it is appended to $\mathtt{p}$.

\section{Experiments}

\subsection{Experiment Setup}
\paragraph{Training Set}
We construct the training
corpus for SFT by combining several curated open-source post-training datasets, including
the NVIDIA OpenMathReasoning~\citep{moshkov2025aimo}, NVIDIA OpenScience~\citep{nvidia2025openscience},
NVIDIA OpenCodeReasoning2~\citep{ahmad2025opencodereasoning}, 
OpenThoughts3~\citep{guha2026openthoughts}, and Mixture-of-Thoughts~\citep{open-r1}.

\paragraph{Benchmarks}
We evaluate the models on a comprehensive suite of tasks covering diverse aspects of reasoning abilities. The evaluation suite includes (i) mathematical reasoning tasks: MATH~\citep{Saxton2019Analysing}, AIME 2024~\citep{aime_2024} and AIME 2025~\citep{aime_2025}; (ii) scientific reasoning: GPQA~\citep{rein2023gpqa}; and (iii) code generation tasks: HumanEval, HumanEval+~\citep{chen2021evaluating}, MBPP, and MBPP+~\citep{austin2021program}.
All benchmarks are evaluated in a zero-shot setting. We report accuracy on MATH, AIME 2024, AIME 2025, and GPQA, and Pass@1 on HumanEval, HumanEval+, MBPP, and MBPP+.

\paragraph{Baseline}
We compare \ours{} against 7B--16B models, including MDMs such as Fast-dLLM~\citep{wu2025fast2}, Dream~\citep{ye2025dream}, Nemotron~\citep{fu2026nemotron}, and SDAR~\citep{cheng2025sdar}, as well as hybrid MDM--UDM models such as LLaDA 2.1~\citep{bie2026llada2} and XDLM~\citep{liu2026balancing}. All models use confidence-based sampler with a threshold of $0.9$.

\begin{table*}[t]
\centering
\caption{
Results on mathematical reasoning, scientific reasoning, and code generation benchmarks.
Models are grouped by their post-training recipes.
For LED, the green superscripts indicate the absolute improvements over the corresponding MDM-SFT baseline with the same model size.
Avg. denotes the average performance across all eight benchmarks.
The best result in each column is shown in \textbf{bold}, and the second-best result is \underline{underlined}.
}
\label{tab:benchmark_result}

\renewcommand{\arraystretch}{1.35}
\resizebox{\textwidth}{!}{%
\begin{tabular}{lccccccccc}
\toprule

\multirow{2}{*}{Model}
& \multicolumn{3}{c}{Math}
& \multicolumn{1}{c}{Science}
& \multicolumn{4}{c}{Code}
& \multirow{2}{*}{Avg.}
\\

\cmidrule(lr){2-4}
\cmidrule(lr){5-5}
\cmidrule(lr){6-9}

& MATH500
& AIME24
& AIME25
& GPQA
& HumanEval
& HumanEval+
& MBPP
& MBPP+
&
\\

\midrule

\multicolumn{10}{l}{\textit{CPT+SFT+RL}} \\

LLaDA 2.1
16B
& \textbf{92.8}
& \textbf{50.0}
& \underline{36.7}
& \underline{50.5}
& \underline{84.8}
& \underline{80.5}
& \textbf{80.4}
& \textbf{67.2}
& \underline{67.9}
\\

\midrule

\multicolumn{10}{l}{\textit{CPT + SFT}} \\

SDAR
8B
& 68.8
& 10.0
& 6.7
& 33.3
& 73.8
& 68.9
& 75.4
& 62.2
& 49.9
\\

Nemotron
8B
& 86.6
& 20.0
& 26.7
& 39.4
& 82.3
& 79.3
& \underline{77.0}
& \underline{64.6}
& 59.5
\\

Dream
7B
& 39.2
& 0.0
& 0.0
& 26.8
& 48.8
& 43.9
& 67.7
& 55.8
& 35.3
\\

\midrule

\multicolumn{10}{l}{\textit{SFT}} \\
XDLM
8B
& 30.2
& 3.3
& 0.0
& 29.8
& 28.0
& 23.8
& 43.4
& 34.4
& 24.1
\\

Fast-dLLM
7B
& 58.0
& 0.0
& 6.7
& 31.9
& 63.4
& 58.5
& 63.0
& 52.3
& 41.7
\\

MDM-SFT
7B
& 72.0
& 13.3
& 13.3
& 22.7
& 62.8
& 59.8
& 64.0
& 55.0
& 45.4
\\

\ours{}
7B
& \scoregain{86.0}{14.0}
& \scoregain{26.7}{13.3}
& \scoregain{23.3}{10.0}
& \scoregain{33.8}{11.1}
& \scoregain{75.0}{12.2}
& \scoregain{70.7}{10.9}
& \scoregain{74.6}{10.6}
& \scoregain{64.0}{9.0}
& 56.8
\\

MDM-SFT
8B
& 88.8
& \underline{36.7}
& 23.3
& 46.5
& 79.9
& 74.4
& 74.6
& 61.9
& 60.8
\\

\ours{}
8B
& \scoregain{\underline{92.6}}{3.8}
& \scoregain{\textbf{50.0}}{13.3}
& \scoregain{\textbf{40.0}}{16.7}
& \scoregain{\textbf{53.6}}{7.1}
& \scoregain{\textbf{87.2}}{7.3}
& \scoregain{\textbf{84.1}}{9.7}
& \scoregain{75.4}{0.8}
& \scoregain{65.1}{3.2}
& \textbf{68.5}
\\
\bottomrule
\end{tabular}
}
\end{table*}

\subsection{Comparison with Existing dLLMs}
We initialize our models from Qwen2.5-7B and Qwen3-8B and perform \ours{} SFT. All models are trained with a block size of 32, a global batch size of 512, a sequence length of 16,384, and a constant learning rate of $10^{-5}$ with Adam optimizer.
Compared with SFT-only baselines, including
Fast-dLLM (41.7 average accuracy) and XDLM (24.1), \ours{}-\texttt{8B} achieves an average score of 68.5, outperforming them by \(+26.8\) and \(+44.4\) points, respectively.
We further compare \ours{} with dLLMs trained using more extensive post-training pipelines.
Among models trained with continued pre-training (CPT) followed by SFT, Nemotron-8B achieves the highest average score of 59.5, while SDAR-8B  and Dream-7B achieve 49.9 and 35.3, respectively.
With SFT alone, \ours{}-\texttt{8B} achieves an average score of 68.5, outperforming Nemotron-8B by $+9.0$ points.
Moreover, \ours{}-\texttt{8B} also surpasses LLaDA2.1-16B, which is trained with CPT, SFT, and RL, by $+0.6$ points on average.
These show that \ours{} is highly competitive despite using SFT alone.

Moreover, we also train \ours{} and the standard MDM SFT~\citep{arriola2025block} on the same dataset under identical training settings for a controlled comparison.
As shown in \cref{tab:benchmark_result}, \ours{} consistently improves over the corresponding MDM-SFT baseline across all eight benchmarks and at both model scales.
For the 7B model, \ours{} increases the average score from 45.4 to 56.8, corresponding to an absolute improvement of $+11.4$ points.
The gains are consistent across mathematical reasoning, scientific reasoning, and code generation, with the largest improvement of $+13.4$ points on AIME 2024.
A similar trend is observed for the 8B model, where \ours{} improves the average score from 60.8 to 68.5, yielding a $+7.7$-point gain.
In particular, \ours{} achieves a substantial improvement of $+16.7$ points on AIME 2025.
These results demonstrate that introducing rewritable lookahead proposals consistently improves masked diffusion fine-tuning across different model scales and task domains.

\begin{figure}[t]
    \centering
    \includegraphics[width=0.48\linewidth]{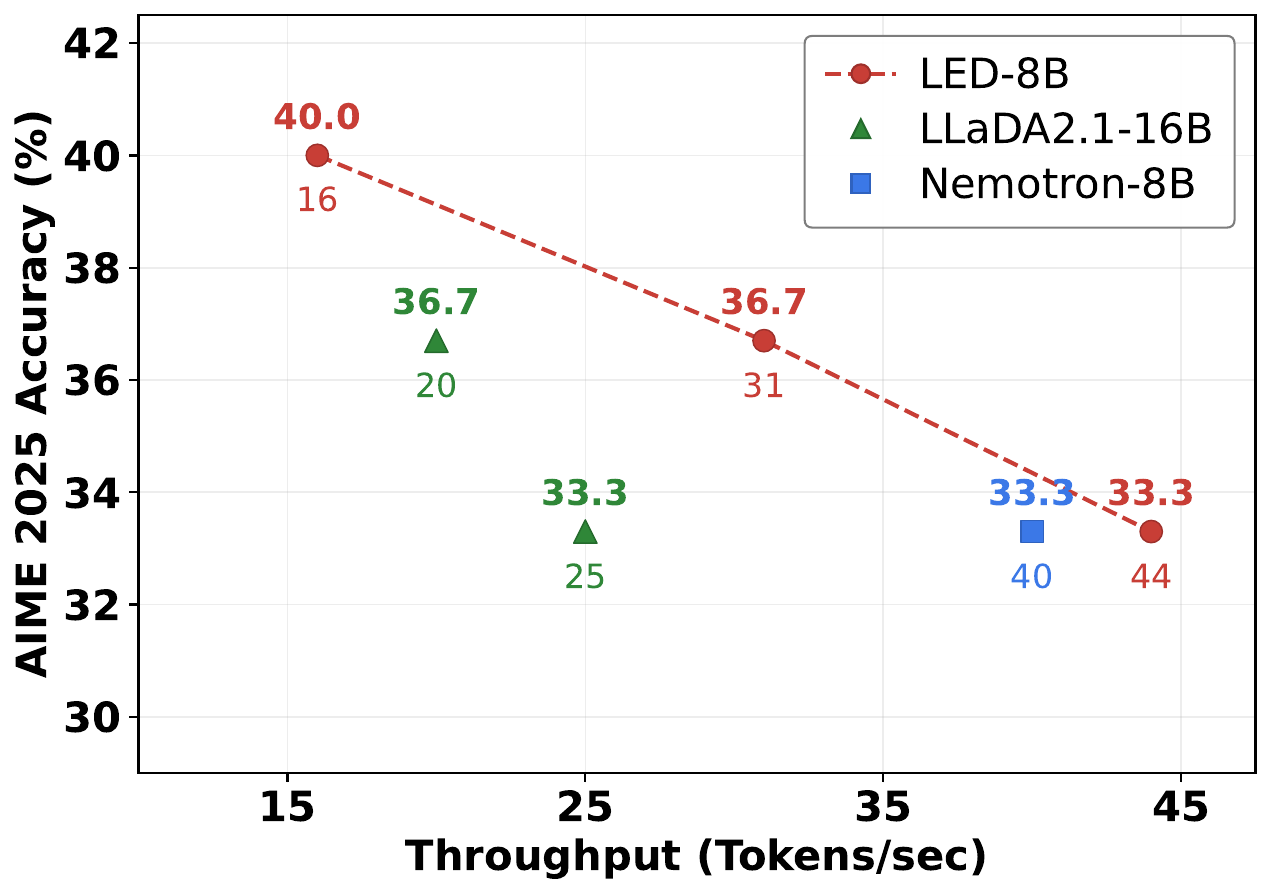}
    \includegraphics[width=0.48\linewidth]{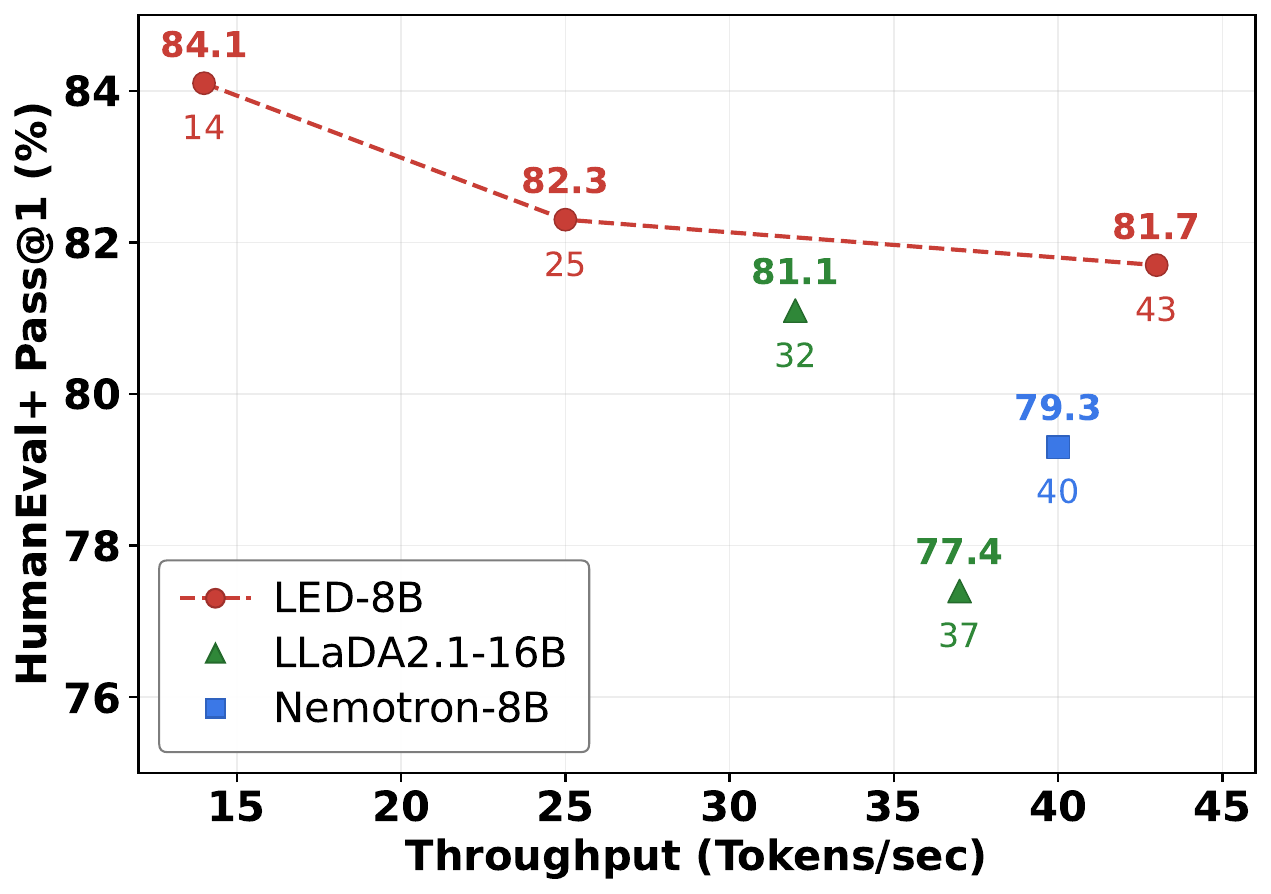}
    \caption{
    Speed--accuracy trade-off of \ours{}-\texttt{8B}, LLaDA2.1-16B, and Nemotron-8B on AIME 2025 and HumanEval+ dataset.
    For \ours{}-\texttt{8B}, different operating points are obtained by varying the number of lookahead steps.
    For LLaDA2.1-16B, we report its Q and S modes.
    \ours{}-\texttt{8B} achieves a stronger Pareto frontier across both benchmarks.
    }
    \label{fig:speed_accuracy}
\end{figure}

\subsection{Speed-Accuracy Trade-off}

\cref{fig:speed_accuracy} compares the speed-accuracy trade-off of \ours{}-\texttt{8B} against LLaDA2.1-16B and Nemotron-8B on AIME 2025 and HumanEval+ datasets.
For \ours{}-\texttt{8B}, we vary the number of lookahead steps from 0 to 2,
yielding different operating points along the accuracy--throughput Pareto front.
For LLaDA2.1-16B, we report results for both its Q and S modes, which are designed to prioritize generation quality and speed, respectively.
For Nemotron-8B, we report the accuracy obtained using greedy threshold decoding.

On AIME 2025 dataset, \ours{}-\texttt{8B} achieves accuracy--throughput pairs of
$(40.0, 16)$, $(36.7, 31)$, and $(33.3, 44)$, where throughput is measured
in tokens per second.
In comparison, LLaDA2.1-16B achieves $(36.7, 20)$ and $(33.3, 25)$ in its
Q and S modes, respectively.
At matched accuracy, \ours{}-\texttt{8B} consistently achieves substantially
higher throughput: at $36.7\%$ accuracy, it reaches $31$ tokens/s compared
with $20$ tokens/s for LLaDA2.1-16B ($1.55\times$ faster), while at
$33.3\%$ accuracy, it reaches $44$ tokens/s compared with $25$ tokens/s
($1.76\times$ faster).
Nemotron-8B achieves its highest accuracy
$33.3\%$ at $40$ tokens/s, which is also
dominated by \ours{}-\texttt{8B}, which attains the same accuracy at
$44$ tokens/s.

A similar trend is observed on HumanEval+.
\ours{}-\texttt{8B} achieves $(84.1, 14)$, $(82.3, 25)$, and $(81.7, 43)$,
whereas LLaDA2.1-16B achieves $(81.1, 32)$ and $(77.4, 37)$.
In particular, at its high-throughput operating point, \ours{}-\texttt{8B}
achieves both higher accuracy and higher throughput than LLaDA2.1-16B:
$81.7\%$ at $43$ tokens/s versus $81.1\%$ at $32$ tokens/s.
It also Pareto-dominates Nemotron-8B, which achieves $79.3\%$ at
$40$ tokens/s.
Overall, these results show that lookahead provides \ours{} with a flexible
mechanism for trading generation speed for accuracy, while yielding a
stronger speed--accuracy Pareto frontier than the compared diffusion
language models across both mathematical reasoning and code generation.

\begin{figure}
    \centering
    \includegraphics[width=0.48\linewidth]{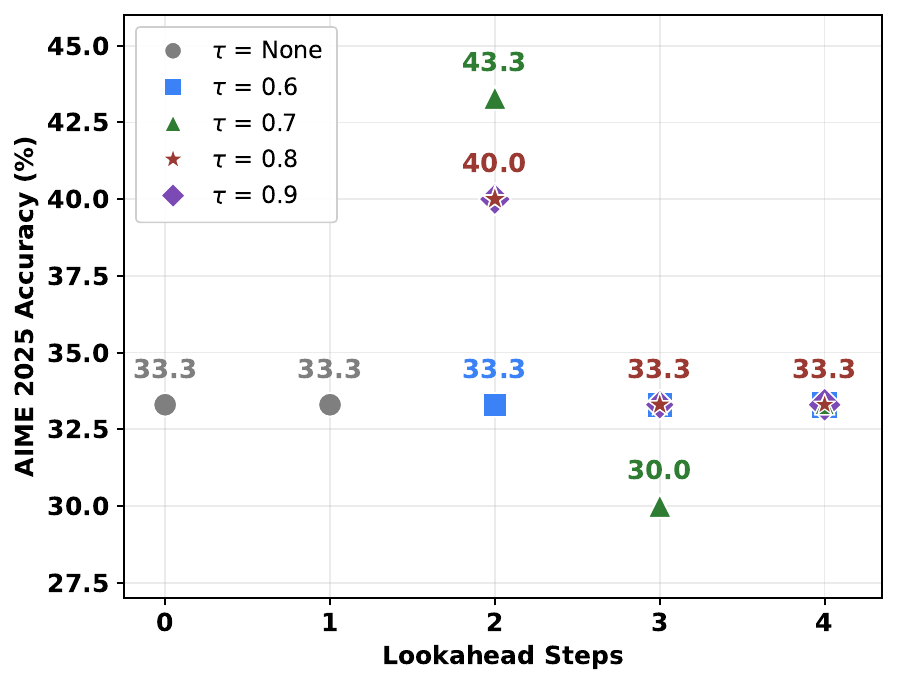}
    \includegraphics[width=0.48\linewidth]{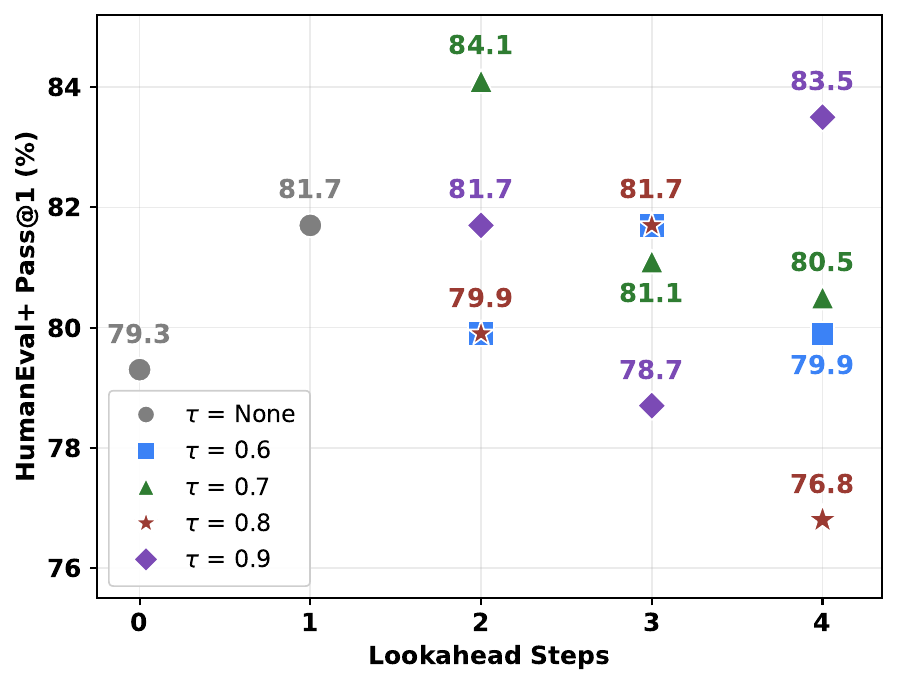}
    \caption{
    Hyper-parameter analysis of the number of lookahead steps \(K\) and proposal threshold \(\tau\) on AIME 2025 and HumanEval+ dataset. A moderate number of lookahead steps (\(K=2\)) generally yields the strongest performance, with \(\tau=0.7\) achieving the best results on AIME 2025 and HumanEval+ benchmarks.  
    }
    \label{fig:hyperparameter}
\end{figure}
\subsection{Hyper-parameter Analysis}

We study the effect of two key hyper-parameters in \ours{}: proposal
threshold $\tau$ and number of 
lookahead steps $K$.
Figure~\ref{fig:hyperparameter}
reports the results on AIME 2025 and HumanEval+.
Overall, the two hyper-parameters jointly control how aggressively the model
exploits its lookahead proposals, and we find that a moderate lookahead times
combined with an appropriate proposal threshold yields the best performance.

\paragraph{Proposal Threshold.}
The proposal threshold $\tau$ controls the acceptance of lookahead proposals.
At a fixed number of lookahead iterations, the performance varies noticeably with $\tau$,
indicating that appropriately filtering proposal tokens is important.
For $K=2$, increasing the threshold from $\tau=0.6$ to $\tau=0.7$
substantially improves AIME 2025 accuracy from $33.3\%$ to $43.3\%$, while
HumanEval+ increases from $79.9\%$ to $84.1\%$.
Further increasing the threshold does not consistently improve performance:
on AIME 2025, $\tau=0.8$ and $\tau=0.9$ both achieve approximately $40.0\%$,
This suggests that an overly aggressive or overly conservative proposal
selection criterion can both be suboptimal, and that an intermediate
threshold provides a better balance between exploiting useful proposals and
filtering unreliable ones.
Among the configurations shown, $\tau=0.7$ with $K=2$ achieves the best
performance on both benchmarks.

\paragraph{Number of Lookahead Steps $K$.}
Without sufficiently deep lookahead, the model obtains relatively limited
gains: on AIME 2025, the accuracy remains at $33.3\%$ for $K=0$ and $K=1$.
Increasing the look ahead times to $K=2$ leads to a substantial improvement,
reaching up to $43.3\%$ on AIME 2025 and $84.1\%$ on HumanEval+.
However, further increasing $K$ does not consistently yield additional
benefits.
On AIME 2025, performance drops considerably for $K=3$ and $K=4$ across
different threshold choices, suggesting that excessively long lookaheads may
accumulate inaccurate proposals and reduce their usefulness for subsequent
generation.
A similar, although less pronounced, trend is observed on HumanEval+.
Interestingly, with a conservative threshold of $\tau=0.9$, $K=4$ still
achieves $83.5\%$, indicating that longer lookahead can remain effective when
low-confidence proposals are sufficiently filtered.
Overall, these results suggest that a moderate look ahead times, particularly
$K=2$, provides a robust balance between obtaining informative future
context and avoiding error accumulation from excessively long proposal
lookaheads.

\section{Conclusion}
In this work, we analyzed the trade-off between optimization simplicity and computational expressivity in diffusion language models.
Our theoretical analysis reveals that: masked diffusion models are easier to optimize due to their simple mask-based transition schedules but have limited iterative revision capability, whereas uniform diffusion models achieve greater computational expressivity through iterative state revision at the cost of more challenging optimization.
Motivated by this trade-off, we propose Lookahead Editing (\ours{}), which augments masked diffusion decoding with a rewritable proposal state.
By repeatedly refining future predictions in the proposal state before committing tokens to the masked sequence, \ours{} introduces iterative revision capability while preserving the simple transition structure of masked diffusion.
Across eight reasoning and code-generation benchmarks, \ours{} consistently improves over standard MDM fine-tuning, with \ours{}-\texttt{8B} further outperforming the more extensively post-trained LLaDA2.1-16B while achieving a stronger speed--accuracy Pareto front than existing dLLM baselines.
Overall, our results demonstrate that \ours{} effectively bridges the gap between the optimization simplicity of masked diffusion and the iterative refinement capability of more expressive diffusion processes.
\bibliographystyle{plainnat}
\bibliography{main}

\clearpage

\beginappendix

\section{Proof of \cref{thm:mdm_where}}

\begin{proof}
For a masked diffusion model, let
$
S(X)
:=
\inf
\left\{
t\in[0,1]:
X_t=m
\right\}
$
denote the transition time from a clean token to the absorbing mask state $m$.
We have
\begin{align}
p\!\left(S(X)\le t\right)
&=
p(X_t=m\mid X_0=x)
 =
1-\alpha_t.
\label{eq:forward_schedule_cdf}
\end{align}
Hence, the cumulative distribution function and the corresponding
probability density function of the forward transition schedule are
\begin{align}
F_p(t)
=
1-\alpha_t,
\quad
p_S(t)
=
\frac{d}{dt}F_p(t)
=
-\dot{\alpha}_t.
\label{eq:forward_schedule_density}
\end{align}

We next derive the transition-time distribution under the reverse
process. For $0\le s<t\le1$, by Bayes' rule,
\begin{align}
p(X_s=m\mid X_t=m,X_0=x)
&=
\frac{
p(X_t=m\mid X_s=m,X_0=x)
p(X_s=m\mid X_0=x)
}{
p(X_t=m\mid X_0=x)
}
\nonumber\\
&=
\frac{1-\alpha_s}{1-\alpha_t},
\label{eq:exact_reverse_posterior}
\end{align}

The standard MDM reverse transition is obtained by marginalizing this
exact posterior over the model prediction of the clean token:
\begin{align}
q_\theta(X_s=m\mid X_t=m)
&=
\sum_{x }
p(X_s=m\mid X_t=m,X_0=x)
f_\theta(x\mid X_t,t)
\nonumber\\
&=
\sum_{x }
\frac{1-\alpha_s}{1-\alpha_t}
f_\theta(x\mid X_t,t)
\nonumber\\
&=
\frac{1-\alpha_s}{1-\alpha_t}
\sum_{x }
f_\theta(x\mid X_t,t)
\nonumber\\
&=
\frac{1-\alpha_s}{1-\alpha_t},
\label{eq:reverse_mask_probability}
\end{align}

Since $\alpha_1=0$, we have:
\begin{align}
p(X_1=m\mid X_0=x)
&=
1-\alpha_1
=
1.
\label{eq:terminal_mask_probability}
\end{align}
Setting the later time in \cref{eq:reverse_mask_probability} to $1$
therefore gives
\begin{align}
q_\theta(X_t=m\mid X_1=m)
&=
\frac{1-\alpha_t}{1-\alpha_1}
=
1-\alpha_t.
\label{eq:reverse_mask_marginal}
\end{align}

Under the reverse path distribution, the same deterministic map $S(X)$
associates each trajectory with its transition time. Since
$S(X)\le t$ if and only if $X_t=m$, we have
\begin{align}
q_\theta\!\left(S(X)\le t\right)
&=
q_\theta(X_t=m\mid X_1=m)
=
1-\alpha_t.
\label{eq:reverse_schedule_cdf}
\end{align}
Thus, the cumulative distribution function and probability density
function of the reverse transition schedule are
\begin{align}
F_{q_\theta}(t)
=
1-\alpha_t,
q_{\theta,S}(t)
\quad
\frac{d}{dt}F_{q_\theta}(t)
=
-\dot{\alpha}_t.
\label{eq:reverse_schedule_density}
\end{align}

Comparing
\cref{eq:forward_schedule_density,eq:reverse_schedule_density}, we have
$
p_S(t)
=
q_{\theta,S}(t)
=
-\dot{\alpha}_t,
\;
\forall\,t\in[0,1].
$
Since the two schedule distributions have the same density, the pushforward distributions induced by $S(X)$ under the
forward and reverse path distributions coincide:
$
p(S)
=
q_\theta(S).
$
Consequently,
\begin{align}
\mathcal{L}_{\mathrm{when}}(\theta)
&=
D_{\mathrm{KL}}
\left(
p(S)\,\|\,q_\theta(S)
\right)
=
0.
\end{align}
\end{proof}
\section{Proof of \cref{thm:udm_schedule_matching}}

For a trajectory $X$ in UDMs,
define its transition schedule as
\[
S(X)=\{t_1,\ldots,t_M\},
\qquad
0<t_1<\cdots<t_M<1,
\]
where each $t_i$ satisfies
$
X_{t_i^-}\neq X_{t_i}.
$
The UDM forward kernel is
$
Q_{s,t}
=
\alpha_{s,t} I
+
(1-\alpha_{s,t})
\frac{1}{V}\mathbf 1\mathbf 1^\top,
$
where
$
\alpha_t := \alpha_{0,t},
\;
\alpha_{s,t} = \frac{\alpha_t}{\alpha_s}.
$
Define the forward noise rate
$
\beta_t
:=
-\frac{d}{dt}\log \alpha_t.
$
It follows that
$
\dot \alpha_t = -\beta_t \alpha_t.
$
For an infinitesimal interval of length $h>0$,
\begin{align*}
\alpha_{t,t+h}
=
\frac{\alpha_{t+h}}{\alpha_t}
=\frac{\alpha_t+\dot\alpha_t h+o(h)}{\alpha_t}
&= 1-\beta_t h+o(h).
\end{align*} 

For any $y\neq x$,
\[
\begin{aligned}
\Pr(X_{t+h}=y\mid X_t=x)
=
\frac{1-\alpha_{t,t+h}}{V}
=
\frac{\beta_t}{V}h+o(h).
\end{aligned}
\]
Therefore, the forward jump rate is
$r_t(x,y)
=
\frac{\beta_t}{V},
 y\neq x.
$
Summing over all possible destination states gives the total forward jump
rate
\[
\begin{aligned}
\lambda^*(t)
&=
\sum_{y\neq x} r_t(x,y)
 =
\frac{V-1}{V}\beta_t.
\end{aligned}
\]

We now derive the induced schedule density. Let
$
R(a,b)
$ denote the 
probability that no transition occurs on $(a,b]$.
By the definition of the total jump rate,
\[
R(a,b+h)
=
R(a,b)
\left[
1-\lambda^*(b)h+o(h)
\right].
\]
Therefore,
\[
\frac{\partial}{\partial b}R(a,b)
=
-\lambda^*(b)R(a,b),
\]
with initial condition $R(a,a)=1$. Solving this differential equation gives
\[
R(a,b)
=
\exp
\left(
-\int_a^b \lambda^*(u)\,du
\right).
\]

Consequently, for a schedule
$
S=\{t_1,\ldots,t_M\},
\;
0<t_1<\cdots<t_M<1,
$
its density is
\[
\begin{aligned}
p(S)
&= 
R(0,t_1)\lambda^*(t_1)
R(t_1,t_2)\lambda^*(t_2)
\cdots
R(t_{M-1},t_M)\lambda^*(t_M)
R(t_M,1)\\
&= \exp
\left(
-\int_0^1 \lambda^*(t)\,dt
\right)
\prod_{i=1}^M
\lambda^*(t_i).
\end{aligned}
\]

We first derive the reverse jump rate induced by the denoiser.  
For $y\neq x$, Bayes' rule gives
\[
\begin{aligned}
&
\Pr(
X_{t-h}=y
\mid
X_t=x,X_0=x_0
)
\\
&=
\frac{
\Pr(
X_t=x
\mid
X_{t-h}=y,X_0=x_0
)
\Pr(
X_{t-h}=y
\mid
X_0=x_0
)
}{
\Pr(
X_t=x
\mid
X_0=x_0
)
}.
\end{aligned}
\]
By the Markov property,
\begin{align}
\Pr(
X_t=x
\mid
X_{t-h}=y,X_0=x_0
)
=
\Pr(
X_t=x
\mid
X_{t-h}=y
) =
\frac{\beta_t}{V}h+o(h),
\end{align}
while
\begin{align}
\Pr(X_{t-h}=y \mid X_0=x_0)
&=
\alpha_{t-h}\mathbf{1}\{y=x_0\}
+
\frac{1-\alpha_{t-h}}{V}
\\
&=
\left(\alpha_t-h\dot{\alpha}_t+o(h)\right)
\mathbf{1}\{y=x_0\}
+
\frac{
1-\alpha_t+h\dot{\alpha}_t+o(h)
}{V}
\\
&=
\Pr(X_t=y \mid X_0=x_0)
-
h\dot{\alpha}_t
\left(
\mathbf{1}\{y=x_0\}
-
\frac{1}{V}
\right)
+
o(h)
\end{align}
Hence
\[
\begin{aligned}
\Pr(
X_{t-h}=y
\mid
X_t=x,X_0=x_0
)=
\frac{\beta_t}{V}
\frac{
\Pr(X_t=y\mid X_0=x_0)
}{
\Pr(X_t=x\mid X_0=x_0)
}
h
+
o(h).
\end{aligned}
\]
Therefore, the exact reverse rate conditioned on $X_0=x_0$ is
\[
r_t^{\leftarrow}(x,y\mid x_0)
=
\frac{\beta_t}{V}
\frac{
\Pr(X_t=y\mid X_0=x_0)
}{
\Pr(X_t=x\mid X_0=x_0)
}.
\]

The reverse model does not observe $X_0$. Let $
f_\theta(x_0\mid x,t)
$
denote the denoiser prediction of the clean-state posterior given $X_t=x$.
Marginalizing over $x_0$ gives the model reverse rate:
\[
\begin{aligned}
r_t^\theta(x,y)
&=
\sum_{x_0\in\mathcal V}
f_\theta(x_0\mid x,t)
r_t^{\leftarrow}(x,y\mid x_0)
\\
&=
\frac{\beta_t}{V}
\sum_{x_0}
f_\theta(x_0\mid x,t)
\frac{
\Pr(X_t=y\mid X_0=x_0)
}{
\Pr(X_t=x\mid X_0=x_0)
}.
\end{aligned}
\]

The total reverse jump rate is therefore
$
\lambda_{f_\theta}(x,t)
:=
\sum_{y\neq x}
r_t^\theta(x,y).
$
Substituting the previous expression gives
\[
\lambda_{f_\theta}(x,t)
=
\frac{\beta_t}{V}
\sum_{x_0}
f_\theta(x_0\mid x,t)
\frac{
\sum_{y\neq x}
\Pr(X_t=y\mid X_0=x_0)
}{
\Pr(X_t=x\mid X_0=x_0)
}.
\]

There are two cases.

If $x_0=x$, then
\[
\Pr(X_t=x\mid X_0=x)
=
\alpha_t+\frac{1-\alpha_t}{V}
=
\frac{1+(V-1)\alpha_t}{V},
\]
while
\[
\sum_{y\neq x}
\Pr(X_t=y\mid X_0=x)
=
\frac{(V-1)(1-\alpha_t)}{V}.
\]
Thus
\[
\frac{
\sum_{y\neq x}\Pr(X_t=y\mid X_0=x)
}{
\Pr(X_t=x\mid X_0=x)
}
=
\frac{1-\alpha_t}
{\frac{1}{V-1}+\alpha_t}.
\]
If $x_0\neq x$, then
\[
\Pr(X_t=x\mid X_0=x_0)
=
\frac{1-\alpha_t}{V},
\]
and therefore
\[
\sum_{y\neq x}
\Pr(X_t=y\mid X_0=x_0)
=
\frac{V-1+\alpha_t}{V}.
\]
Hence
\[
\frac{
\sum_{y\neq x}\Pr(X_t=y\mid X_0=x_0)
}{
\Pr(X_t=x\mid X_0=x_0)
}
=
\frac{V-1+\alpha_t}{1-\alpha_t}.
\]

Finally,
\[ 
\lambda_{f_\theta}(x,t)
=
\frac{\beta_t}{V}
\left[
f_\theta(x\mid x,t)
\frac{1-\alpha_t}
{\frac{1}{V-1}+\alpha_t}
+
\bigl(1-f_\theta(x\mid x,t)\bigr)
\frac{V-1+\alpha_t}
{1-\alpha_t}
\right].
\]

To express the reverse process in increasing time, define
$
\tau := 1-t,
\;
Y_\tau := X_{1-\tau}.
$
The reverse generative process then evolves along
$
\tau:0\rightarrow1.
$
Define its state-conditioned jump rate by
$
\mu_\theta(x,\tau)
:=
\lambda_{f_\theta}(x,1-\tau).
$
Let
$
N_\tau
$ denote the schedule counting process of reverse transitions in $[0,\tau]$, and let
$
\mathcal F_{\tau^-}^{S}
$
denote the history states generated only by transition times strictly before $\tau$.

Conditioned on the current state,
\[
\Pr(
N_{\tau+h}-N_\tau=1
\mid
Y_\tau=x,\mathcal F_{\tau^-}^{S}
)
=
\mu_\theta(x,\tau)h+o(h).
\]
Marginalizing the latent state gives
\[
\begin{aligned}
\Pr(
N_{\tau+h}-N_\tau=1
\mid
\mathcal F_{\tau^-}^{S}
)
&=
\sum_x
\Pr(
N_{\tau+h}-N_\tau=1
\mid
Y_\tau=x,\mathcal F_{\tau^-}^{S}
)
q_\theta(
Y_\tau=x
\mid
\mathcal F_{\tau^-}^{S}
)
\\
&=
\left[
\sum_x
\mu_\theta(x,\tau)
q_\theta(
Y_\tau=x
\mid
\mathcal F_{\tau^-}^{S}
)
\right]
h
+
o(h).
\end{aligned}
\]
Therefore, the conditional intensity of the marginal schedule process is
\[
\bar\mu_\theta
(\tau\mid\mathcal F_{\tau^-}^{S})
=\lim_{h\to 0} \frac{\left[
\sum_x
\mu_\theta(x,\tau)
q_\theta(
Y_\tau=x
\mid
\mathcal F_{\tau^-}^{S}
)
\right]
h
+
o(h)}{h} = 
\mathbb E_{q_\theta}
\left[
\mu_\theta(Y_\tau,\tau)
\mid
\mathcal F_{\tau^-}^{S}
\right].
\]
Equivalently,
\[ 
\bar\mu_\theta
(\tau\mid\mathcal F_{\tau^-}^{S})
=
\mathbb E_{q_\theta}
\left[
\lambda_{f_\theta}
(Y_\tau,1-\tau)
\mid
\mathcal F_{\tau^-}^{S}
\right]. 
\]

Let
$
\Sigma
=
\{\tau_1,\ldots,\tau_M\},
\;
0<\tau_1<\cdots<\tau_M<1,
$
where
$\tau_i
=
1-t_{M+1-i}.$
Let $
R_\theta(a,b)
$
denote the conditional probability that no transition occurs on $(a,b]$
along the observed schedule history. By the definition of
$\bar\mu_\theta$,
\[
R_\theta(a,b+h)
=
R_\theta(a,b)
\left[
1-
\bar\mu_\theta
(b\mid\mathcal F_{b^-}^{S})
h
+
o(h)
\right].
\]
By definition of partial derivative, we have:
\[
\frac{\partial}{\partial b}R_\theta(a,b)
=
-
\bar\mu_\theta
(b\mid\mathcal F_{b^-}^{S})
R_\theta(a,b),
\]
and, since $R_\theta(a,a)=1$,
\[
R_\theta(a,b)
=
\exp
\left[
-\int_a^b
\bar\mu_\theta
(u\mid\mathcal F_{u^-}^{S})
\,du
\right].
\]

The density of observing transitions exactly at
$\tau_1,\ldots,\tau_M$ and no additional transitions is therefore
\[
\begin{aligned}
q_\theta(\Sigma)
&={}
R_\theta(0,\tau_1)
\bar\mu_\theta(
\tau_1\mid\mathcal F_{\tau_1^-}^{S})
\times
R_\theta(\tau_1,\tau_2)
\bar\mu_\theta(
\tau_2\mid\mathcal F_{\tau_2^-}^{S})
\times\cdots
R_\theta(\tau_M,1)\\
&=
\exp
\left[
-\int_0^1
\bar\mu_\theta
(\tau\mid\mathcal F_{\tau^-}^{S})
\,d\tau
\right]
\prod_{i=1}^{M}
\bar\mu_\theta
(\tau_i\mid\mathcal F_{\tau_i^-}^{S}).
\end{aligned}
\]

We now express the forward schedule in the same reverse-time coordinate.
Define
\[
\mu^*(\tau)
:=
\lambda^*(1-\tau)
=
\frac{V-1}{V}\beta_{1-\tau}.
\]
Since the transformation $t\mapsto1-t$ is bijective with unit absolute
Jacobian,
\[
D_{\mathrm{KL}}
\bigl(
p(S)\,\|\,q_\theta(S)
\bigr)
=
D_{\mathrm{KL}}
\bigl(
p(\Sigma)\,\|\,q_\theta(\Sigma)
\bigr).
\]
Moreover,
\[ 
p(\Sigma)
=
\exp
\left[
-\int_0^1 \mu^*(\tau)\,d\tau
\right]
\prod_{i=1}^{M}
\mu^*(\tau_i).
\]

We therefore have
\[
\begin{aligned}
\mathcal L_{\mathrm{when}}(\theta)
&=
\mathbb E_{\Sigma\sim p}
\left[
\log
\frac{p(\Sigma)}
{q_\theta(\Sigma)}
\right].
\end{aligned}
\]
Substituting the two schedule densities gives
\[
\begin{aligned}
\log
\frac{p(\Sigma)}
{q_\theta(\Sigma)}
=
\sum_{i=1}^M
\log
\frac{
\mu^*(\tau_i)
}{
\bar\mu_\theta(
\tau_i\mid\mathcal F_{\tau_i^-}^{S})
}
 +
\int_0^1
\left[
\bar\mu_\theta(
\tau\mid\mathcal F_{\tau^-}^{S})
-
\mu^*(\tau)
\right]d\tau.
\end{aligned}
\]

It remains to convert the sum over observed transition times into an
integral. For any predictable function $g(\tau)$,
\[
\sum_{i=1}^M g(\tau_i)
=
\int_0^1 g(\tau)\,dN_\tau.
\]
Under the forward schedule distribution $p$, the counting process has
intensity $\mu^*(\tau)$, and thus
\[
\mathbb E_p[
dN_\tau
\mid
\mathcal F_{\tau^-}^{S}
]
=
\mu^*(\tau)\,d\tau.
\]
Using the tower property,
\[
\begin{aligned}
\mathbb E_p
\left[
\int_0^1 g(\tau)\,dN_\tau
\right]
=
\mathbb E_p
\left[
\int_0^1
g(\tau)
\mathbb E_p[
dN_\tau
\mid
\mathcal F_{\tau^-}^{S}
]
\right]
=
\mathbb E_p
\int_0^1
\mu^*(\tau)g(\tau)\,d\tau.
\end{aligned}
\]
Taking
$
g(\tau)
=
\log
\frac{
\mu^*(\tau)
}{
\bar\mu_\theta(
\tau\mid\mathcal F_{\tau^-}^{S})
}
$
yields
\[
\begin{aligned}
\mathbb E_p
\left[
\sum_{i=1}^M
\log
\frac{
\mu^*(\tau_i)
}{
\bar\mu_\theta(
\tau_i\mid\mathcal F_{\tau_i^-}^{S})
}
\right]
=
\mathbb E_p
\int_0^1
\mu^*(\tau)
\log
\frac{
\mu^*(\tau)
}{
\bar\mu_\theta(
\tau\mid\mathcal F_{\tau^-}^{S})
}
\,d\tau.
\end{aligned}
\]

Combining the two terms gives the schedule-matching loss:
 
\begin{align}
\mathcal L_{\mathrm{when}}(\theta)
=
\mathbb E_{\Sigma\sim p}
\int_0^1
\Bigg[
\mu^*(\tau)
\log
\frac{
\mu^*(\tau)
}{
\bar\mu_\theta(
\tau\mid\mathcal F_{\tau^-}^{S})
} 
-
\mu^*(\tau)
+
\bar\mu_\theta(
\tau\mid\mathcal F_{\tau^-}^{S})
\Bigg]
d\tau.
\end{align}

\section{Proof of \cref{thm:sgd_convergence}}
For each sample when loss is
\begin{align*}
 {\ell}_{\mathrm{when}}(\theta)
=
\mu^*
\log
\frac{\mu^*}{\lambda_\theta}
-
\mu^*
+
\lambda_\theta.
\label{eq:sample_surrogate_loss}
\end{align*}
The gradient w.r.t $\lambda_{\theta}$ is
\[
\frac{\partial \ell}{\partial \lambda}
=
1-\frac{\mu^*}{\lambda}.
\]
Therefore, at the optimum $\theta^\star$, we have
\begin{align*}
\lambda_{\theta^\star}
=
\mu^*,\; 
1-\frac{\mu^*}{\lambda_{\theta^\star}}
=
0,
\end{align*}
The gradient w.r.t $ {\theta}$ is
\begin{align*}
\nabla_\theta
 {\ell}_{\mathrm{when}}(\theta)
=
-\mu^*
\nabla_\theta
\log \lambda_\theta
+
\nabla_\theta\lambda_\theta
=
\left(
1-\frac{\mu^*}{\lambda_\theta}
\right)
\nabla_\theta\lambda_\theta.
\label{eq:surrogate_gradient}
\end{align*}

The Hessian is,
\begin{align*}
\nabla_\theta^2
 {\ell}_{\mathrm{when}}(\theta)
={}&
\frac{\mu^*}{\lambda_\theta^2}
\nabla_\theta\lambda_\theta
\nabla_\theta\lambda_\theta^\top
+
\left(
1-\frac{\mu^*}{\lambda_\theta}
\right)
\nabla_\theta^2\lambda_\theta.
\end{align*}

Therefore
\begin{align*}
\nabla_\theta^2
 {\ell}_{\mathrm{when}}(\theta^\star)
=
\frac{1}{\mu^*}
\nabla_\theta\lambda_{\theta^\star}
\nabla_\theta\lambda_{\theta^\star}^{\top}.
\end{align*}

Taking expectation,
\begin{align*}
H_{\mathrm{when}}
&:=
\nabla^2
 {\mathcal L}_{\mathrm{when}}
(\theta^\star)
=
\mathbb E_{S,\tau,Y_\tau}
\left[
\frac{1}{\mu^*(\tau)}
\nabla_\theta
\lambda_{f_{\theta^\star}}
\nabla_\theta
\lambda_{f_{\theta^\star}}^\top
\right]. 
\end{align*}

For any $v\in\mathbb R^d$,
\begin{align*}
v^\top H_{\mathrm{when}}v
&=
\mathbb E
\left[
\frac{1}{\mu^*}
v^\top
\nabla\lambda_{\theta^\star}
\nabla\lambda_{\theta^\star}^\top
v
\right]
 =
\mathbb E
\left[
\frac{
\left(
v^\top
\nabla\lambda_{\theta^\star}
\right)^2
}{
\mu^*
}
\right]
 \ge0.
\end{align*}
Therefore,
$
H_{\mathrm{when}}\succeq0
$ is positive semidefinite.
Since
$\mathcal L_U
=
\mathcal L_M
+
 {\mathcal L}_{\mathrm{when}},
$
we have
$
H_U
=
H_M
+
H_{\mathrm{when}}.
$
Therefore,
$
H_U\succeq H_M.
$

We next characterize the stochastic-gradient covariance.
For each model $j\in\{M,U\}$, let $Z$ denote one stochastic
training sample and let
$
\ell_j(\theta;Z)
=
-\log q_\theta^j(Z)
$
denote the corresponding sample negative log-likelihood.
Define the sample gradient
$
g_j(\theta;Z)
=
\nabla_\theta \ell_j(\theta;Z),
$
and let
$
s_j(\theta;Z)
=
\nabla_\theta
\log q_\theta^j(Z)
$
denote the score function. Then
$
g_j(\theta;Z)
=
-s_j(\theta;Z).
$
For any $\theta$, the score has zero mean under the model
distribution:
\begin{align}
\mathbb E_{Z\sim q_\theta^j}
\left[
s_j(\theta;Z)
\right]
&=
\int
q_\theta^j(z)
\nabla_\theta
\log q_\theta^j(z)
\,dz
\\
&=
\int
\nabla_\theta
q_\theta^j(z)
\,dz
\\
&=
\nabla_\theta
\int
q_\theta^j(z)
\,dz
\\
&=
\nabla_\theta
1
\\
&=
0.
\label{eq:score_zero}
\end{align}

Differentiating \cref{eq:score_zero} w.r.t $\theta$
gives
\begin{align*}
0
&=
\nabla_\theta
\int
q_\theta^j(z)
s_j(\theta;z)
\,dz
\\
&=
\int
\nabla_\theta q_\theta^j(z)
s_j(\theta;z)^\top
\,dz
+
\int
q_\theta^j(z)
\nabla_\theta
s_j(\theta;z)^\top
\,dz.
\end{align*}

Using
$
\nabla_\theta q_\theta^j(z)
=
q_\theta^j(z)
s_j(\theta;z),
$
we have 
\begin{align*}
\mathbb E_{q_\theta^j}
\left[
s_j(\theta;Z)
s_j(\theta;Z)^\top
\right]
=
-
\mathbb E_{q_\theta^j}
\left[
\nabla_\theta^2
\log q_\theta^j(Z)
\right].
\end{align*}

At the optimum $\theta^\star$, the single-sample stochastic-gradient covariance is
\begin{align*}
\operatorname{Cov}
\left[
g_j(\theta^\star;Z)
\right]
=\mathbb{E}[g_j g_j^\top ] - \mathbb{E}[g_j]\mathbb{E}[g_j ]  ^\top =\mathbb{E}[g_j g_j^\top ] = 
H_j. 
\end{align*}

For a mini-batch of $B$ iid samples,
\begin{align*}
\widehat g_j
=
\frac{1}{B}
\sum_{b=1}^{B}
g_j^{(b)}.
\end{align*}
By independence,
\begin{align*}
\operatorname{Cov}
\left[
\widehat g_j
\right]
&=
\frac{1}{B^2}
\sum_{b=1}^{B}
\operatorname{Cov}
\left[
g_j^{(b)}
\right]
\\
&=
\frac{1}{B^2}
BH_j
\\
&=
\frac{1}{B}H_j.
\end{align*}
Since the loss of UDM and MDM has the form of negative log-likelihood, the local gradient-noise covariance $\Sigma_j$ is
$
\Sigma_j
=
\frac{1}{B}H_j,
\;
j\in\{M,U\}.
$
Since
$
H_U
=
H_M+H_{\mathrm{when}},
$
we have
\begin{align*}
\Sigma_U-\Sigma_M
&=
\frac{1}{B}
(H_U-H_M)
\\
&=
\frac{1}{B}
H_{\mathrm{when}}
\succeq0.
\end{align*}
Thus,
$
\Sigma_U\succeq\Sigma_M.
$

Consider either model $j\in\{M,U\}$ and define
$
e_k^j
=
\theta_k^j-\theta^\star.
$
Under the second order Taylor approximation,
\begin{align*}
\mathcal L_j(\theta)
=
\mathcal L_j^\star
+
\frac12
(\theta-\theta^\star)^\top
H_j
(\theta-\theta^\star), 
\end{align*}
and therefore
\begin{align*}
\nabla_\theta
\mathcal L_j(\theta_k^j)
=
H_je_k^j.
\end{align*}

Write the stochastic gradient as
\begin{align*}
\widehat g_j(\theta_k^j)
=
H_je_k^j+\xi_k^j,
\end{align*}
where
$
\mathbb E
\left[
\xi_k^j
\mid
e_k^j
\right]
=
0.
$
The SGD update is
\begin{align*}
\theta_{k+1}^j
=
\theta_k^j
-
\eta
\widehat g_j(\theta_k^j).
\end{align*}
Subtracting $\theta^\star$ from both sides yields
\begin{align}
e_{k+1}^j
=
(I-\eta H_j)e_k^j
-
\eta\xi_k^j.
\label{eq:error_dynamics}
\end{align}

Define
$
A_j
=
I-\eta H_j.
$
Since
$
H_U\succeq H_M,
$
we have
\begin{align*}
\lambda_{\max}(H_U)
\ge
\lambda_{\max}(H_M).
\end{align*}
Therefore, the common learning-rate condition
$
0<\eta<
\frac{2}{\lambda_{\max}(H_U)}
$
implies
$
0<\eta<
\frac{2}{\lambda_{\max}(H_j)}
$
for both $j\in\{M,U\}$.

Diagonalize
$
H_j
=
Q_j\Lambda_jQ_j^\top,
$
where
$
\Lambda_j
=
\operatorname{diag}
(
\lambda_{j,1},
\ldots,
\lambda_{j,d}
).
$
Then
\begin{align*}
A_j
=
Q_j
(I-\eta\Lambda_j)
Q_j^\top.
\end{align*}

For every eigenvalue $\lambda_{j,i}>0$,
$
0
<
\eta\lambda_{j,i}
<
2,
$
which implies
$
|1-\eta\lambda_{j,i}|<1.
$
Hence
$
\rho(A_j)<1,
$
and therefore
$
A_j^k\to0
\;
\text{as }
k\to\infty.
$

Define the second-moment matrix
$
P_k^j
=
\mathbb E
\left[
e_k^j(e_k^j)^\top
\right].
$
From \cref{eq:error_dynamics}, we have
\begin{align*}
e_{k+1}^j(e_{k+1}^j)^\top
={}&
A_je_k^j(e_k^j)^\top A_j^\top
 -
\eta
A_je_k^j(\xi_k^j)^\top
 -
\eta
\xi_k^j(e_k^j)^\top A_j^\top
 +
\eta^2
\xi_k^j(\xi_k^j)^\top.
\end{align*}

Taking expectation, the cross terms vanish because
\begin{align*}
\mathbb E
\left[
e_k^j(\xi_k^j)^\top
\right]
&=
\mathbb E
\left[
e_k^j
\mathbb E
\left[
(\xi_k^j)^\top
\mid
e_k^j
\right]
\right]
=
0.
\end{align*}

Thus
\begin{align*} 
P_{k+1}^j
=
A_jP_k^jA_j^\top
+
\eta^2\Sigma_j. 
\end{align*}
Iterating it gives
\begin{align*}
P_\infty^j
=
\eta^2
\sum_{r=0}^{\infty}
A_j^r
\Sigma_j
(A_j^\top)^r. 
\end{align*}

Using
$
\Sigma_j
=
\frac{1}{B}H_j
=
\frac{1}{B}
Q_j\Lambda_jQ_j^\top,
$
and
$
A_j
=
Q_j
(I-\eta\Lambda_j)
Q_j^\top,
$
we have
\begin{align*}
P_\infty^j
=
Q_j
\left[
\frac{\eta^2}{B}
\sum_{r=0}^{\infty}
(I-\eta\Lambda_j)^r
\Lambda_j
(I-\eta\Lambda_j)^r
\right]
Q_j^\top.
\end{align*}

The $i$-th diagonal entry in the eigenbasis is
\begin{align*}
p_{\infty,j,i}
&=
\frac{\eta^2\lambda_{j,i}}{B}
\sum_{r=0}^{\infty}
(1-\eta\lambda_{j,i})^{2r}
=
\frac{
\eta^2\lambda_{j,i}/B
}{
1-(1-\eta\lambda_{j,i})^2
}.
\end{align*}

Since
\begin{align*}
1-(1-\eta\lambda_{j,i})^2
&=
1-
\left(
1-2\eta\lambda_{j,i}
+\eta^2\lambda_{j,i}^2
\right)
\\
&=
2\eta\lambda_{j,i}
-\eta^2\lambda_{j,i}^2
\\
&=
\eta\lambda_{j,i}
(2-\eta\lambda_{j,i}),
\end{align*}
we have
\begin{align*}
{
p_{\infty,j,i}
=
\frac{\eta}
{B(2-\eta\lambda_{j,i})}.
}
\end{align*}

Equivalently,
\begin{align}
{
P_\infty^j
=
\frac{\eta}{B}
(2I-\eta H_j)^{-1}.
}
\label{eq:P_closed_form}
\end{align}

Under the second order Taylor approximation,
\begin{align*}
\mathcal L_j(\theta_k^j)
-
\mathcal L_j^\star
=
\frac12
(e_k^j)^\top
H_j
e_k^j.
\end{align*}
Therefore,
\begin{align*}
\mathbb E
\left[
\mathcal L_j(\theta_k^j)
-
\mathcal L_j^\star
\right]
&=
\frac12
\operatorname{Tr}
\left(
H_jP_k^j
\right).
\end{align*}

Taking $k\to\infty$ and using
\eqref{eq:P_closed_form},
\begin{align}
\mathbb E_{\mathrm{ss}}
\left[
\mathcal L_j-\mathcal L_j^\star
\right]
=
\frac12
\operatorname{Tr}
\left(
H_jP_\infty^j
\right)
=
{
\frac{\eta}{2B}
\operatorname{Tr}
\left[
H_j
(2I-\eta H_j)^{-1}
\right].
}
\label{eq:stationary_excess}
\end{align}


We have already established
$
H_U\succeq H_M.
$
Multiplying by the negative scalar $-\eta$ reverses the
Loewner order:
\begin{align*}
-\eta H_U
\preceq
-\eta H_M.
\end{align*}
Adding $2I$ gives
\begin{align}
2I-\eta H_U
\preceq
2I-\eta H_M.
\label{eq:X_order}
\end{align}

For positive-definite matrices, inversion reverses the Loewner
order. Therefore, \eqref{eq:X_order} implies
\begin{align}
{
(2I-\eta H_U)^{-1}
\succeq
(2I-\eta H_M)^{-1}.
}
\label{eq:inverse_order}
\end{align}

Define
$
F(H)
=
H(2I-\eta H)^{-1}.
$
Using
$
\eta H
=
2I-(2I-\eta H),
$
we have
\begin{align*}
\eta H(2I-\eta H)^{-1}
&=
2(2I-\eta H)^{-1}
-
I,
\end{align*}
and hence
\begin{align}
{
F(H)
=
\frac{2}{\eta}
(2I-\eta H)^{-1}
-
\frac{1}{\eta}I.
}
\label{eq:F_identity}
\end{align}

Combining \eqref{eq:inverse_order}
and \eqref{eq:F_identity},
\begin{align}
F(H_U)
\succeq
F(H_M).
\end{align}
Therefore,
\begin{align}
\operatorname{Tr}
\left[
H_U(2I-\eta H_U)^{-1}
\right]
\ge
\operatorname{Tr}
\left[
H_M(2I-\eta H_M)^{-1}
\right].
\end{align}

Using \eqref{eq:stationary_excess}, we have
\begin{align}
{
\mathbb E_{\mathrm{ss}}
\left[
\mathcal L_U-\mathcal L_U^\star
\right]
\ge
\mathbb E_{\mathrm{ss}}
\left[
\mathcal L_M-\mathcal L_M^\star
\right].
}
\end{align}

\section{Proof of \cref{thm:mdm_experss}}
We first show that replacing causal self-attention with block-causal
self-attention does not increase the circuit complexity of one forward
pass. \citet{li2024chain} show that a constant-depth decoder-only
Transformer with hidden dimension $\operatorname{poly}(n)$, numerical
precision $O(\log n)$, and zero exponent bits can be simulated by a
polynomial-size constant-depth threshold circuit. Their analysis
considers causal self-attention, where the attention weights at position
$i$ are
\[
s_i^{\mathrm{causal}}
=
\operatorname{softmax}_{0,s}
\left(
\langle q_i,k_1\rangle_{0,s},
\ldots,
\langle q_i,k_i\rangle_{0,s}
\right)
\Vert 0^{n-i}.
\]

Our MDM uses block-causal self-attention. For any $i\in[n]$, let
$r(i)\coloneqq
\min\{B(n)\lceil i/B(n)\rceil,n\}$ be the right endpoint of the block
containing position $i$. Then the attention weights at position $i$ are
\[
s_i^{\mathrm{blk}}
=
\operatorname{softmax}_{0,s}
\left(
\langle q_i,k_1\rangle_{0,s},
\ldots,
\langle q_i,k_{r(i)}\rangle_{0,s}
\right)
\Vert 0^{n-r(i)}.
\]
Thus the only difference is that the causal prefix endpoint $i$ is
replaced by $r(i)$. Since $r(i)\le n$, the same finite-precision
inner-product, softmax, and weighted-summation operations considered by
\citet{li2024chain} are applied to at most $n$ inputs of $O(\log n)$
bits. Thus the same $\mathsf{TC}^0$ simulation applies to block-causal
self-attention. In particular, one forward pass of the MDM can be
computed by a polynomial-size constant-depth threshold circuit.

We next adapt the parallel-enumeration argument of
\citet{li2024chain} to MDM decoding. Let $\Gamma_B$ be the set of
block-local decoding states needed to specify one decoding step, where
each state specifies the token/mask configuration of the active block
and its block-local generation order. Since $B=O(1)$ and
$|\mathcal V|=O(1)$, we have
$|\Gamma_B|\le (|\mathcal V|+1)^B B!=O(1)$.

A $T$-step MDM trajectory can therefore be represented by
$\tau=(\gamma_1,\ldots,\gamma_T)\in\Gamma_B^T$. For
$T\le c\log n$, the number of candidate trajectories is at most
$|\Gamma_B|^T\le |\Gamma_B|^{c\log n}
=n^{c\log |\Gamma_B|}=\operatorname{poly}(n)$.

We enumerate all candidate trajectories $\tau\in\Gamma_B^T$ in
parallel. Fix a candidate
$\tau=(\gamma_1,\ldots,\gamma_T)$. For every
$t\in\{0,\ldots,T\}$, let $X_\tau^{(t)}(x)$ denote the MDM state
obtained from $x$ after applying the first $t$ candidate actions
$\gamma_1,\ldots,\gamma_t$, where $X_\tau^{(0)}(x)=x$. For every
$t\in[T]$, define
\[
E_{\tau,t}(x)
\coloneqq
\mathbf{1}
\left[
F_{\theta,n}
\bigl(X_\tau^{(t-1)}(x),t\bigr)
=
\gamma_t
\right].
\]
Thus $E_{\tau,t}(x)=1$ if and only if the $t$-th candidate action is
consistent with the MDM decoding rule.

For a fixed $\tau$, the actions
$\gamma_1,\ldots,\gamma_T$ are fixed. Hence
$X_\tau^{(t-1)}(x)$ is determined by $x$ and
$(\gamma_1,\ldots,\gamma_{t-1})$, and does not require the preceding
model evaluations to be carried out sequentially. Thus all
$E_{\tau,t}$ can be evaluated in parallel.

By the above single-forward-pass argument,
$F_{\theta,n}(X_\tau^{(t-1)}(x),t)$ can be computed by a
polynomial-size constant-depth threshold circuit. Since $\gamma_t$
has constant size, testing equality with $\gamma_t$ also takes
constant depth. Thus $E_{\tau,t}\in\mathsf{TC}^0$ for every
$\tau\in\Gamma_B^T$ and $t\in[T]$.

For each candidate trajectory $\tau$, define
$E_\tau(x)\coloneqq\bigwedge_{t=1}^T E_{\tau,t}(x)$. Since
$\mathsf{TC}^0$ allows unbounded-fan-in AND gates,
$E_\tau\in\mathsf{TC}^0$. Under a deterministic decoding rule with
fixed tie-breaking, for every input $x$ there is a unique trajectory
$\tau^\star$ such that $E_{\tau^\star}(x)=1$.

Let $G_\tau(x)$ denote the final output corresponding to candidate
trajectory $\tau$. For every output bit $j$, define
\[
C_{n,j}(x)
\coloneqq
\bigvee_{\tau\in\Gamma_B^T}
\left(
E_\tau(x)\wedge G_{\tau,j}(x)
\right).
\]
Since there are only $\operatorname{poly}(n)$ candidate trajectories
and $\mathsf{TC}^0$ allows unbounded-fan-in OR gates,
$C_{n,j}\in\mathsf{TC}^0$. Moreover, since $\tau^\star$ is the unique
consistent trajectory, $C_{n,j}(x)=G_{\tau^\star,j}(x)$. Thus $C_n$
computes exactly the output of the original $T$-step MDM. This
completes the proof.
\section{Proof of \cref{thm:udm_experss}}

Let
$
f=\{f_n:\{0,1\}^{n}\to\{0,1\}\}_{n\geq 1}
\in\mathsf{NC}^{1}
$
be arbitrary.
By Barrington's theorem, for every $n$, there exists a width-$5$
permutation branching program
$
\mathcal{B}_n=(I_1,\ldots,I_{T(n)})
$
of length
$
T(n)=\operatorname{poly}(n)
$
that computes $f_n$.
Each instruction is of the form
$
I_t=(i_t,\pi_t^0,\pi_t^1),
$
where
$
i_t\in[n]
$
and
$
\pi_t^0,\pi_t^1\in S_5
$.
For an input
$
x=(x_1,\ldots,x_n)\in\{0,1\}^n,
$
define
$
\rho_t(x)\triangleq \pi_t^{x_{i_t}},
$
where
$
\pi_t^{x_{i_t}}=\pi_t^0
$
if
$
x_{i_t}=0
$
and
$
\pi_t^{x_{i_t}}=\pi_t^1
$
.
The branching program can be chosen such that, for some fixed
$
\alpha\in S_5,
$
\[
f_n(x)=1
\quad\Longleftrightarrow\quad
\rho_1(x)\circ\cdots\circ\rho_{T(n)}(x)=\alpha.
\]

We construct a constant-depth UDM that simulates one instruction of
$
\mathcal{B}_n
$
at each decoding step.
Define
$
\sigma_0\triangleq e
$
and
$
\sigma_t\triangleq\sigma_{t-1}\circ\rho_t(x)
$
for
$
t\in[T(n)],
$
where $e$ is the identity element of $S_5$.
It therefore suffices to maintain
$
\sigma_t
$
in a rewritable workspace.

Let the vocabulary contain the input symbols
$
\{0,1\}
$
and one workspace symbol
$
\langle \sigma\rangle
$
for each
$
\sigma\in S_5
$.
For every
$
\sigma\in S_5,
$
let
$
e_\sigma\in\{0,1\}^{120}
$
denote its one-hot encoding.
Since
$
|S_5|=120,
$
all representations associated with permutations have constant
dimension.

Let
$
k\triangleq \lceil \log_2(n+2)\rceil
$
and use
$
\operatorname{sbin}_k(j)
$
in \citep{li2024chain} to denote the signed binary representation of an address $j$.
We reserve address
$
n+1
$
for the workspace.
Let $C$ be a sufficiently large constant and take the hidden dimension
to be
$
d'=2k+C
$.
We partition the coordinates of each hidden representation into blocks
storing, respectively, the input value, the position address, the query
address, the current permutation state, the two permutations in the
current instruction, the retrieved input bit, the next permutation
state, a constant $1$, and auxiliary coordinates.

For every input position
$
j\in[n],
$
we choose the token and positional embeddings such that
\[
h_j^0
=
\left[
x_j,\,
\operatorname{sbin}_k(j),\,
0,\,
0,\,
0,\,
0,\,
0,\,
0,\,
1,\,
0
\right]^\top,
\]
where each zero denotes the zero vector of the corresponding block.

Suppose that before decoding step $t$, the workspace token is
$
\langle\sigma_{t-1}\rangle
$.
Its token embedding places
$
e_{\sigma_{t-1}}
$
in the current-state block.
The externally supplied timestep embedding at step $t$ encodes the
instruction
$
I_t=(i_t,\pi_t^0,\pi_t^1)
$
by placing
$
\operatorname{sbin}_k(i_t),
e_{\pi_t^0},
$
and
$
e_{\pi_t^1}
$
in the query-address and instruction blocks.
Hence the initial hidden representation at the workspace position,
denoted by $*$, is
\[
h_*^0
=
\left[
0,\,
\operatorname{sbin}_k(n+1),\,
\operatorname{sbin}_k(i_t),\,
e_{\sigma_{t-1}},\,
e_{\pi_t^0},\,
e_{\pi_t^1},\,
0,\,
0,\,
1,\,
0
\right]^\top.
\]

We next construct the attention layer.
Following the exact-addressing construction of
\citet{li2024chain},
we choose fixed linear projections
$
W_Q,W_K,W_V
$
such that $W_Q$ reads the query-address and constant-$1$ blocks,
$W_K$ reads the position-address and constant-$1$ blocks, and $W_V$
reads the input-value coordinate.
In particular,
\[
q_*
\triangleq
W_Q h_*^0
=
\operatorname{sbin}_k(i_t)
\frown
\mathbf{1}_k,
\]
while, for every
$
j\in[n],
$
\[
k_j
\triangleq
W_K h_j^0
=
B_s
\left(
\operatorname{sbin}_k(j)
\frown
(-\mathbf{1}_k)
\right),
\qquad
v_j
\triangleq
W_V h_j^0
=
x_j e_b,
\]
where $e_b$ denotes a designated retrieved-bit coordinate.
At the workspace position,
\[
k_*
=
W_K h_*^0
=
B_s
\left(
\operatorname{sbin}_k(n+1)
\frown
(-\mathbf{1}_k)
\right),
\qquad
v_*=0.
\]
These projections are realizable by fixed linear maps since they only
select, rescale, duplicate, and relocate coordinates of the hidden
representation.

By the exact-addressing lemma of
\citet{li2024chain},
\[
\left[
\exp
\left(
\left\langle q_*,k_j\right\rangle_s
\right)
\right]_s
=
\mathbf{1}[j=i_t]
\]
for every
$
j\in[n].
$
Moreover,
$
i_t\neq n+1
$,
so the workspace position itself receives zero attention weight.
Hence the attention vector at the workspace position selects exactly
position
$
i_t
$,
and its attention output is
\[
W_O
\sum_j
a_jv_j
=
x_{i_t}e_b,
\]
where $W_O$ places the result in the retrieved-bit coordinate.
After the residual connection, the workspace representation
$
h_*^{0.5}
$
therefore contains
$
e_{\sigma_{t-1}},
e_{\pi_t^0},
e_{\pi_t^1},
$
and
$
x_{i_t}
$.

We now construct the feed-forward layer implementing the
branching-program transition.
For
$
\sigma,\mu,\nu\in S_5
$
and
$
c\in\{0,1\},
$
let
$
u_\sigma,
v_\mu^0,
$
and
$
v_\nu^1
$
denote the corresponding coordinates of the current-state,
$\pi^0$, and $\pi^1$ one-hot blocks of
$
h_*^{0.5}
$.
Let $b$ denote the retrieved-bit coordinate and define
$
r_0(b)\triangleq 1-b
$
and
$
r_1(b)\triangleq b
$.
For every
$
(\sigma,\mu,\nu,c)\in S_5^3\times\{0,1\},
$
define
\[
z_{\sigma,\mu,\nu,c}
\triangleq
\operatorname{ReLU}
\left(
u_\sigma
+
v_\mu^0
+
v_\nu^1
+
r_c(b)
-3
\right).
\]
Since all four quantities inside the sum are binary,
\[
z_{\sigma,\mu,\nu,c}
=
\mathbf{1}
\left[
\sigma=\sigma_{t-1},
\,
\mu=\pi_t^0,
\,
\nu=\pi_t^1,
\,
c=x_{i_t}
\right].
\]
Thus exactly one such hidden unit is equal to $1$.

For
$
c\in\{0,1\},
$
define
\[
P_c(\mu,\nu)
\triangleq
\begin{cases}
\mu, & c=0,\\
\nu, & c=1.
\end{cases}
\]
For every
$
\tau\in S_5,
$
define the corresponding next-state coordinate by
\[
y_\tau
\triangleq
\sum_{\substack{
\sigma,\mu,\nu\in S_5,\,
c\in\{0,1\}:\\
\sigma\circ P_c(\mu,\nu)=\tau
}}
z_{\sigma,\mu,\nu,c}.
\]
Since exactly one
$
z_{\sigma,\mu,\nu,c}
$
is nonzero,
\[
y_\tau
=
\mathbf{1}
\left[
\tau
=
\sigma_{t-1}
\circ
\pi_t^{x_{i_t}}
\right].
\]
Therefore the next-state block is exactly
\[
y
=
e_{
\sigma_{t-1}
\circ
\pi_t^{x_{i_t}}
}
=
e_{\sigma_t}.
\]
The above mapping is realizable by a two-layer ReLU network of constant
width, since
$
2|S_5|^3
$
is independent of $n$.
Moreover, all weights appearing in this construction belong to a fixed
finite set and are therefore exactly representable under
$
O(\log n)
$
bits of fixed-point precision.

We choose the output projection so that, for every
$
\tau\in S_5,
$
the logit of the workspace symbol
$
\langle\tau\rangle
$
is the coordinate
$
y_\tau
$.
Hence
$
\langle\sigma_t\rangle
$
has logit $1$, whereas every other workspace symbol has logit $0$.
Under deterministic argmax decoding, the model therefore outputs
\[
w_t
=
\langle\sigma_t\rangle,
\]
which rewrites the previous workspace token using the same workspace
location.

We now verify the construction inductively.
For every
$
t\in[T(n)],
$
the workspace position satisfies
\[
h_*^0
=
\left[
0,\,
\operatorname{sbin}_k(n+1),\,
\operatorname{sbin}_k(i_t),\,
e_{\sigma_{t-1}},\,
e_{\pi_t^0},\,
e_{\pi_t^1},\,
0,\,
0,\,
1,\,
0
\right]^\top,
\]
and, after the attention layer,
\[
h_*^{0.5}
=
\left[
0,\,
\operatorname{sbin}_k(n+1),\,
\operatorname{sbin}_k(i_t),\,
e_{\sigma_{t-1}},\,
e_{\pi_t^0},\,
e_{\pi_t^1},\,
x_{i_t},\,
0,\,
1,\,
0
\right]^\top,
\]
while the next-state block after the feed-forward layer is
$
e_{\sigma_t}
$.

We claim that
\[
\sigma_t
=
\rho_1(x)
\circ\cdots\circ
\rho_t(x)
\]
for every
$
0\leq t\leq T(n)
$.
The claim is immediate for
$
t=0
$
since
$
\sigma_0=e
$.
Suppose it holds at step $t-1$.
At step $t$, the timestep embedding supplies
$
(i_t,\pi_t^0,\pi_t^1)
$.
By the attention construction above, the retrieved-bit coordinate is
exactly
$
x_{i_t}
$.
The feed-forward layer therefore computes
\[
\sigma_t
=
\sigma_{t-1}
\circ
\pi_t^{x_{i_t}}
=
\sigma_{t-1}
\circ
\rho_t(x).
\]
Using the induction hypothesis,
\[
\sigma_t
=
\rho_1(x)
\circ\cdots\circ
\rho_{t-1}(x)
\circ
\rho_t(x),
\]
which establishes the claim.

After
$
T(n)
$
decoding steps,
\[
\sigma_{T(n)}
=
\rho_1(x)
\circ\cdots\circ
\rho_{T(n)}(x).
\]
By the defining property of the branching program,
$
f_n(x)=1
$
if and only if
$
\sigma_{T(n)}=\alpha
$.
Since
$
S_5
$
has constant size, this final equality test is realizable by a
constant-size output layer.

Finally,
$
T(n)=\operatorname{poly}(n)
$
by Barrington's theorem.
The signed-binary address blocks require
$
k=O(\log n)
$
coordinates, whereas all permutation-related blocks have constant
dimension.
Thus the construction in fact uses
$
d'(n)=O(\log n)
$
hidden dimension, and hence certainly
$
d'(n)=\operatorname{poly}(n)
$.
It uses
$
O(\log n)
$
bits of fixed-point precision and no exponent bits.
Only one workspace token is preserved and rewritten between decoding
steps, so
$
W(n)=O(1)
$.
Therefore
\[
f
\in
\T_{\mathrm{UDM}}
\left[
\operatorname{poly}(n),
\operatorname{poly}(n),
O(\log n),
0,
O(1)
\right].
\]
Since
$
f\in\mathsf{NC}^{1}
$
was arbitrary, the result follows.

\end{document}